\documentclass[11pt]{article}

\usepackage[final]{acl}

\usepackage{times}
\usepackage{latexsym}
\usepackage{amsmath}
\usepackage{amssymb}
\usepackage{amsthm}
\usepackage{booktabs}
\usepackage{multirow}
\usepackage{graphicx}
\usepackage{enumitem}
\usepackage{bm}
\usepackage{marvosym}

\newtheorem{theorem}{Theorem}[section]

\newtheorem{lemma}[theorem]{Lemma}
\newtheorem{corollary}[theorem]{Corollary}
\newtheorem{assumption}[theorem]{Assumption}

\newtheorem{remark}[theorem]{Remark}

\newcommand{\E}{\mathbb{E}}
\newcommand{\R}{\mathbb{R}}
\newcommand{\1}{\mathbf{1}}
\newcommand{\dd}{\,\mathrm{d}}
\newcommand{\asympconst}{\asymp}
\newcommand{\mb}[1]{\boldsymbol{#1}}
\newcommand{\good}{\mathrm{good}}
\newcommand{\bad}{\mathrm{bad}}
\newcommand{\eff}{\mathrm{eff}}

\usepackage[T1]{fontenc}
\usepackage[utf8]{inputenc}
\usepackage{microtype}
\usepackage{inconsolata}
\usepackage{hyperref}

\newcommand{\email}[1]{\href{mailto:#1}{\normalfont\small #1}}
\newcommand{\correspondingauthor}{\textsuperscript{\Letter}}

\title{How Should LLMs Consume High-Quality Data? Optimal Data Scheduling via Quality-Aware Functional Scaling Laws}

\author{Zhitao Zhu\thanks{This work was done during the internship at Meituan.} \\
  Peking University \\
  \email{zhuzt@pku.edu.cn} \\\And
  Xili Wang\footnotemark[1]\correspondingauthor \\
  Peking University \\
  \email{xiliwang@stu.pku.edu.cn} \\\And
  Shizhe Wu \\
  Meituan \\
  \email{wushizhe@meituan.com} \\\AND
  Jiawei Fu\correspondingauthor \\
  Meituan \\
  \email{fujiawei06@meituan.com} \\\And
  Xiaoqing Liu \\
  Meituan \\
  \email{liuxiaoqing02@meituan.com}}

\begin{document}
\maketitle
\begingroup
\renewcommand{\thefootnote}{\Letter}
\footnotetext[0]{Corresponding authors: Xili Wang and Jiawei Fu.}
\endgroup

\begin{abstract}
High-quality data is scarce in large language model (LLM) training, yet how to schedule its use with optimization dynamics lacks theoretical guidance. We extend functional scaling laws with time-varying data quality and derive asymptotically optimal joint data-quality and batch-size schedules within a feature-space regression model. The solution reveals two regimes and dual uses of high-quality data: in the noise-limited regime, a smaller batch converts cleaner data into more signal at comparable noise; in the signal-limited regime, late placement suppresses terminal noise without sacrificing signal accumulation. This explains why conventional decay schedules can conflict with curriculum-style pipelines. Motivated by the theoretical structure, we propose Drop-Stable-Rampup for LLM midtraining: drop the batch size at the quality transition, keep it low to accumulate signal, then ramp up to suppress noise. On a 15B MoE model midtrained on 108B tokens of general-domain proprietary data, Drop-Stable-Rampup improves average accuracy over Warmup-Stable-Decay by +1.70 and Cosine-decay by +2.98, including +4.23 on GSM8K and +2.80 on MATH. On a public math-and-code mixture, it leads all reported STEM, mathematics, and code benchmarks, improving the overall mean over the strongest baseline by +3.27 on a 600M dense model and +5.25 on the same MoE architecture.
\end{abstract}

\section{Introduction}
\label{sec:intro}

High-quality data, such as textbooks, mathematical proofs and expert-curated code, is the scarcest resource in LLM training. Multi-stage training has become the standard paradigm: models are pretrained on massive noisy web data, then midtrained on a limited curated corpus to sharpen capabilities such as mathematical reasoning, code generation and agentic capabilities~\citep{grattafiori2024llama,yang2025qwen3,kimi2026k25}. Given a fixed budget of high-quality tokens, a fundamental question arises: \textbf{how should LLMs consume high-quality data?} Specifically, at what training stage should it be introduced, and what batch-size or learning-rate schedule should accompany the transition?

A natural intuition from curriculum learning~\citep{bengio2009curriculum} suggests placing high-quality data late in training, when the model has sufficient capacity to exploit it. However, \citet{luo2026how} identified a key conflict: conventional schedules simultaneously decay the learning rate in later stages, so the model encounters its best data precisely when its learning intensity (effective update magnitude per token) is weakest. This reveals that data scheduling and training dynamics cannot be optimized in isolation---they must be considered jointly. Yet existing theoretical frameworks do not characterize this coupled scheduling problem.

We fill this gap by developing a \emph{quality-aware functional scaling law} (Section~\ref{sec:quality-fsl}) that extends \citet{li2026functional} to time-varying data quality. In our framework, the excess risk decomposes into a signal-learning term and a noise-accumulation integral that is jointly controlled by the data-quality schedule $p(t)$ and the batch-size schedule $b(t)$. We formulate the joint optimization as a variational problem and derive asymptotic closed-form optimal schedules (Theorem~\ref{thm:joint-optimal}). The solution exhibits two regimes---\emph{noise-limited} and \emph{signal-limited}---depending on whether the learning bottleneck is noise accumulation or signal acquisition. Together, the two regimes reveal two complementary uses of high-quality data: late placement and terminal batch-size ramp-up use it as a \emph{noise suppressor}, while pairing it with low batch size (or higher learning intensity) uses it as a \emph{signal amplifier}. Existing curriculum-style schedules mainly exploit the former and miss the latter.

% Despite their different mathematical structures, both regimes share a common principle: \emph{dropping the batch size} at the quality transition accumulates signal (equivalently raising learning intensity and resolving the curriculum-versus-decay conflict), while \emph{late placement of high-quality data} and \emph{terminal batch-size ramp-up} suppress noise.

Translating these theoretical insights into practice, we propose \textbf{Drop-Stable-Rampup} (DSR) for LLM midtraining (Section~\ref{sec:exp-llm}): upon the quality transition, drop the batch size; hold it stable at a low level to accumulate signal; then ramp up to suppress noise. Midtraining is a natural testbed because its limited corpus size makes low batch sizes feasible---a prerequisite that pretraining's hardware-utilization constraints typically prevent. On a 15B-total/2.4B-active MoE model trained for 108B tokens, DSR improves average accuracy by +1.70 over WSD~\citep{hu2024minicpm} and +2.98 over Cosine-decay~\citep{loshchilov2017sgdr}. We additionally validate DSR on a 600M dense model and a 2.4B-active MoE model trained on the same public pretraining and midtraining corpora, improving the respective means over the strongest baseline by +3.27 and +5.25 points.

Our contributions are threefold: we derive a quality-aware functional scaling law for time-varying data quality; solve the resulting joint data-quality and batch-size scheduling problem in asymptotic closed form, revealing a noise/signal dichotomy that resolves the curriculum-versus-decay conflict; and translate the theory into Drop-Stable-Rampup, which improves a 15B MoE model across 14 benchmarks and transfers across dense and MoE architectures in a controlled public-data study.

% ======================================================================
\section{Preliminaries}
\label{sec:fsl}
% ======================================================================

\paragraph{Notation.} For two nonnegative functions $f,g$ of an index variable, we write $f\asympconst g$ if there exist constants $C_1,C_2>0$, independent of the index variable, such that $C_1 f\leq g\leq C_2 f$ in the relevant range.

\subsection{Feature-space linear regression}

Let \(\mb{x}\in \mathcal X\) be sampled from a distribution \(\mathcal D\).
Let \(\mb{\phi}:\mathcal X\to\R^N\) be a feature map such that
\[
    \mb{\phi}(\mb{x})\sim \mathcal N(\mb{0},\mb{H}),
\]
where \(\mb{H}\in\R^{N\times N}\) is the feature covariance matrix. We consider
the feature-space linear regression model
\[
    f^\star(\mb{x})=\langle \mb{\phi}(\mb{x}),\mb{\theta}^\star\rangle,
    \quad
    f(\mb{x};\mb{\theta})=\langle \mb{\phi}(\mb{x}),\mb{\theta}\rangle.
\]
Without loss of generality, we assume that \(\mb{H}\) is diagonalized. Let
\(\lambda_1\geq\lambda_2\geq\cdots\geq 0\) denote its eigenvalues, ordered
non-increasingly. This assumption is due to the
rotation equivariance of stochastic gradient descent (SGD).

\begin{assumption}[Power-law capacity and source conditions]
\label{ass:power-law}
There exist \(\beta>1\) and \(s>0\) such that
\[
    \lambda_j\asympconst j^{-\beta}, \quad 
    |\theta_j^\star|^2
    \asympconst
    j^{-1}\lambda_j^{s-1}.
\]
\end{assumption}

The capacity exponent \(\beta\) controls the decay rate of the eigenvalues.
The source exponent \(s\) measures the alignment of the target function with
the eigenstructure: smaller \(s\) corresponds to harder learning problems.
These capacity and source conditions are standard in the analysis of kernel
methods and have recently been adopted in scaling-law
studies~\citep{paquette2024phases,lin2024scaling,bordelon2025feature,li2026functional}.

\subsection{Two-level data-quality model}

We assume that the label is generated as
\[
    y=f^\star(\mb{x})+\epsilon,
\]
where \(\epsilon\) is a zero-mean noise term that captures label corruption.
In practice, training data are drawn from heterogeneous sources of varying
reliability. For instance, in large-scale language model training, curated
corpora (e.g., textbooks, encyclopedias) produce labels that are closely
aligned with the target distribution, whereas noisy web-crawled text introduces
substantially larger deviations. We model this quality heterogeneity through
the variance of the label noise: high-quality data correspond to small noise
variance, and low-quality data correspond to large noise variance.

Concretely, the label noise \(\epsilon\) comes from a two-component mixture:
\[
    \epsilon
    \sim
    \begin{cases}
        \mathcal N(0,\sigma_{\good}^2),
        & \text{high-quality data},\\[3pt]
        \mathcal N(0,\sigma_{\bad}^2),
        & \text{low-quality data},
    \end{cases}
\]
where
\[
    0<\sigma_{\good}^2<\sigma_{\bad}^2.
\]
Let \(p(t)\in[0,1]\) be the proportion of high-quality data used at continuous training time \(t\).
Then the effective label-noise variance at time \(t\) is
\[
    \sigma_{\eff}^2(t)
    :=
    p(t)\sigma_{\good}^2
    +
    (1-p(t))\sigma_{\bad}^2.
\]

\subsection{One-pass SGD}

The parameter \(\mb{\theta}\) is trained by minimizing the population risk
\[
    R(\mb{\theta})
    :=
    \frac12\E_{\mb{x},y}
    \bigl[
        (f(\mb{x};\mb{\theta})-y)^2
    \bigr]
\]
via SGD in the one-pass regime: each data
point is seen exactly once. We fix a constant learning rate \(\eta>0\) and allow
the batch size to vary across iterations. Denote the \emph{batch-size schedule}
by \((B_1,B_2,\dots,B_K)\). We write \(t=k\eta\) for the continuous training
time and \(b(t)\) for the continuous interpolation of \(B_k\), i.e.,
\(b(k\eta)=B_k\). At iteration \(k\), a mini-batch of \(B_k\) samples is
drawn and the parameter is updated as
\begin{equation}
\label{eq:sgd-update}
\mb{\theta}_{k+1}
=
\mb{\theta}_k
-
\frac{\eta}{B_k}
\sum_{i=1}^{B_k}
\nabla_{\mb{\theta}}
\frac12
\bigl(f(\mb{x}_{k,i};\mb{\theta}_k)-y_{k,i}\bigr)^2.
\end{equation}
This can be decomposed into a population gradient plus noise:
\[
    \mb{\theta}_{k+1}
    =
    \mb{\theta}_k
    -
    \eta
    \bigl(\nabla R(\mb{\theta}_k)+\mb{\xi}_k\bigr),
\]
where the gradient noise \(\mb{\xi}_k\) satisfies \(\E[\mb{\xi}_k]=\mb{0}\) and
\(\E[\mb{\xi}_k\mb{\xi}_k^\top]=\mb{\Sigma}(\mb{\theta}_k)/B_k\). Here
\(\mb{\Sigma}(\mb{\theta})\) denotes the per-sample gradient-noise covariance at
\(\mb{\theta}\). We measure performance by the excess risk
\[
    E(\mb{\theta})
    :=
    R(\mb{\theta})-\frac12\E[\epsilon^2]
    =
    \frac12\|\mb{\theta}-\mb{\theta}^\star\|_{\mb{H}}^2.
\]

Beyond the batch-size schedule, our setting introduces a second control: the
\emph{data-quality schedule} \(p(t)\in[0,1]\,(p_k=p(k\eta))\), governing the fraction of
high-quality samples at each time in training. The joint schedule
\((b(t),p(t))\) is the object we seek to optimize. The discrete update~\eqref{eq:sgd-update} can be modeled by the following It\^o
stochastic differential equation (SDE)~\citep{li2019stochastic,orvieto2019continuous,ankirchner2024comparison}:
\begin{equation}
    \label{eq:sde}
    \dd \bar{\mb{\theta}}_t
    =
    -\nabla R(\bar{\mb{\theta}}_t)\dd t
    +
    \sqrt{
        \frac{\eta}{b(t)}
    }
    \mb{\Sigma}(\bar{\mb{\theta}}_t)^{1/2}\dd \mb{W}_t,
\end{equation}
where \(\mb{W}_t\in\R^N\) is a standard Brownian motion. The noise covariance satisfies \(\mb{\Sigma}(\mb{\theta})\asympconst (E(\mb{\theta})+\sigma_{\eff}^2(t))\mb{H}\), i.e., it is proportional to \(\mb{H}\) with a scalar prefactor depending only on the excess risk and the effective noise level (Lemma~\ref{lem:quality-anisotropic-noise} in Appendix~\ref{app:gradient-noise}). This anisotropic noise alignment has been identified in~\citep{mori2022power,wu2022alignment,wang2023noise}; our contribution is the extension to the quality-mixture setting.

\section{Quality-Aware Functional Scaling Law}
\label{sec:quality-fsl}

\subsection{Quality-aware functional scaling law}
\begin{theorem}[Quality-aware functional scaling law]
\label{thm:quality-fsl}
Suppose Assumption~\ref{ass:power-law} holds. Assume that \(b(t)\geq B_{\min}\geq 1\),
\(p(t)\in[0,1]\) is measurable, and the learning rate \(\eta>0\) is sufficiently
small. Define
\[
    K(t):=(t+1)^{-(2-1/\beta)}.
\]
Then, for sufficiently large \(t\),
\[
    \E[E(\bar{\mb{\theta}}_t)]
    \asympconst
    \underbrace{t^{-s}}_{\substack{\text{signal}\\\text{learning}}}
    +
    \underbrace{\eta
    \int_0^t
    K(t-\tau)
    \frac{\sigma_{\eff}^2(\tau)}{b(\tau)}
    \dd \tau}_{\substack{\text{noise}\\\text{accumulation}}}.
\]
% where
% \[
%     \sigma_{\eff}^2(\tau)
%     =
%     p(\tau)\sigma_{\good}^2
%     +
%     (1-p(\tau))\sigma_{\bad}^2.
% \]
\end{theorem}
The proof is in Appendix~\ref{pfofthm:quality-fsl}, and empirical
verification is given in Appendix~\ref{app:fsl-experiment}.

\begin{remark}[Relation to prior work]
\label{rem:relation}
The above theorem extends the functional scaling law (FSL) of \citet{li2026functional} to the data-quality scheduling setting. In their original formulation, the label-noise variance \(\sigma^2\) is a fixed constant. Our contribution is to replace the constant \(\sigma^2\) by the
time-dependent effective variance \(\sigma_{\eff}^2(t)\), which couples the data-quality schedule \(p(t)\) into the loss dynamics.
% Li et al.~\cite{wang2025fastcatchup} used FSL for batch-size scheduling.
\end{remark}

\subsection{Variational formulation}

The quality-aware FSL (Theorem~\ref{thm:quality-fsl}) suggests minimizing the
end-of-training excess risk over the training horizon \(T\), batch-size
schedule \(b(t)\), and data-quality schedule \(p(t)\).
Let \(D\) be the continuous-time total data budget,
\(\rho\in(0,1)\) the fraction of high-quality data,
\(r:=\sigma_{\good}^2/\sigma_{\bad}^2\in(0,1)\) the variance ratio,
and \(B_{\min}\geq 1\) a hardware lower bound on the batch size
(e.g., the number of parallel devices).
The fully joint scheduling problem is
\begin{equation}
\label{eq:joint_problem}
\boxed{
\begin{aligned}
    \min_{T>0,\,b,\,p}\quad
    &
    T^{-s}
    +
    \eta
    \int_0^T
    K(T-t)
    \frac{\sigma_{\eff}^2(t)}{b(t)}
    \dd t
    \\
    \mathrm{s.t.}\quad
    &
    \int_0^T b(t)\dd t=D,
    \\
    &\int_0^T p(t)b(t)\dd t=\rho D,
    \\
    &
    b(t)\geq B_{\min},
    \quad
    0\leq p(t)\leq 1.
\end{aligned}
}
\end{equation}
Define the critical exponent
\(s_{\mathrm{crit}}:=1-1/\beta\).
As we show below, the structure of the optimal schedule depends sharply on
whether \(s>s_{\mathrm{crit}}\) (noise-limited regime) or
\(s< s_{\mathrm{crit}}\) (signal-limited regime).

\section{Optimal Data Schedule}

\subsection{Optimal data-quality schedule}

We first optimize the data-quality schedule \(p(t)\) with the training horizon
\(T\) and batch-size schedule \(b(t)\) fixed. Since
\(\sigma_{\eff}^2(t)=p(t)\sigma_{\good}^2+(1-p(t))\sigma_{\bad}^2\) is
affine in \(p\) and the \(p\)-independent part of the objective is constant,
the problem reduces to
\[
\begin{aligned}
    \max_{0\leq p\leq 1}\;
    \int_0^T K(T-t)\frac{p(t)}{b(t)}\dd t
    \\
    \mathrm{s.t.}\;
    \int_0^T p(t)b(t)\dd t=\rho D.
\end{aligned}
\]

\begin{theorem}[Optimal quality schedule for fixed batch-size schedule]
\label{thm:fixed-b}
Fix \(T>0\) and \(b(t)\geq B_{\min}\). Define the quality-allocation score
\(S_b(t):=K(T-t)/b(t)^2\).
Then an optimal data-quality schedule is bang-bang:
\[
    p^\star(t)
    =
    \begin{cases}
        1, & S_b(t)>\lambda,\\
        0, & S_b(t)<\lambda,
    \end{cases}
\]
where \(\lambda\) is the unique threshold satisfying
\(\int_0^T p^\star(t)b(t)\dd t=\rho D\)
(fractional values arise only on the level set \(\{t:S_b(t)=\lambda\}\)).
In words, high-quality data should be allocated to the time region where
\(K(T-t)/b(t)^2\) is largest.
\end{theorem}
The proof is given in Appendix~\ref{pfofthm:fixed-b}.

\begin{corollary}[Quality placement depends on batch-size shape]
\label{cor:quality-shape}
The quality-allocation score \(S_b(t)=K(T-t)/b(t)^2\) determines whether
quality placement is \emph{determinate} or \emph{degenerate}:
\begin{enumerate}[label=(\roman*)]
    \item If \(b(t)=D/T\) is constant, then \(S_b(t)\propto K(T-t)\), which is largest near
    \(t=T\). The optimal quality schedule is \textbf{late}: high-quality data
    should be saved for the end of training.
    \item If \(b(t)=C\sqrt{K(T-t)}\) with \(C=D/\int_0^T\sqrt{K(T-t)}\dd t\)
    (the optimal batch-size schedule of \citet{wang2025fastcatchup} in the
    noise-limited regime), then \(S_b(t)=1/C^2\) is constant. The quality placement is \textbf{degenerate}: any allocation satisfying the budget constraint is optimal.
\end{enumerate}
\end{corollary}

Corollary~\ref{cor:quality-shape}(i) implies that under a constant batch size, high-quality data should be placed late---yet conventional training simultaneously decays the learning rate in this phase, reducing the model's ability to exploit the high-quality data. This is consistent with the empirical finding of \citet{luo2026how} that learning rate decay conflicts with curriculum-based data scheduling.

Empirical verification of Corollary~\ref{cor:quality-shape} (constant batch and
\(b(t)=C\sqrt{K(T-t)}\)) is provided in
Appendix~\ref{app:corollary-experiment}.

\subsection{Main theory}
\label{sec:joint}

We now solve the fully joint problem~\eqref{eq:joint_problem}, optimizing $T$, $b(t)$, and $p(t)$ simultaneously.
Write $B:=B_{\min}$, $\gamma:=2-1/\beta$, and define
\[
    \kappa:=\frac{\sigma_{\good}^2\sigma_{\bad}^2}{(1-\rho)\sigma_{\good}^2+\rho\sigma_{\bad}^2}.
\]
Let $I_T:=\int_0^T\!\sqrt{K(u)}\dd u$.

\begin{theorem}[Asymptotically optimal joint schedule]
\label{thm:joint-optimal}
Suppose Assumption~\ref{ass:power-law} holds. The following schedules are asymptotically optimal as $D\to\infty$.

\medskip
\noindent\textbf{(I) Noise-limited regime ($s>1-1/\beta$).}
The constraint $b(t)\geq B$ is asymptotically inactive.
The optimal quality schedule is bang-bang:
\[
    p^\star(t)=\1_S(t)
\]
for a measurable set $S\subset[0,T^\star]$ satisfying
\[
    \int_S\!\sqrt{K(T^\star\!-\!t)}\dd t
    =
    \frac{\rho\kappa}{\sigma_{\good}^2}I_{T^\star}.
\]
Any such $S$ is equally optimal (quality placement is degenerate).
The optimal batch-size schedule is
\[
    b^\star(t)=
    \begin{cases}
        C\sqrt{K(T^\star\!-\!t)}, & t\notin S,\\
        rC\sqrt{K(T^\star\!-\!t)}, & t\in S,
    \end{cases}
\]
where $C=D\sigma_{\bad}^2/(\kappa\, I_{T^\star})$.
The optimal training horizon satisfies
\[
    T^\star
    =
    \left(\frac{sD}{4\beta\eta\kappa}\right)^{\!\frac{1}{s+1/\beta}}
    \!\!(1+o(1)),
\]
and the optimal excess risk satisfies
\[
    \E[E(\bar{\mb{\theta}}_{T^\star})]
    \asympconst
    D^{-\frac{s\beta}{1+s\beta}}.
\]

\medskip
\noindent\textbf{(II) Signal-limited regime ($s<1-1/\beta$).}
The constraint $b(t)\geq B$ is active. The optimal schedule has the strict-switching form:
\begin{enumerate}[nosep]
    \item \textbf{Low-quality, flat}: $p^\star\!=\!0$, $b^\star\!=\!B$ for $t\in[0,\,T_1^\star)$.
    \item \textbf{High-quality, flat}: $p^\star\!=\!1$, $b^\star\!=\!B$ for $t\in[T_1^\star,\,T_1^\star\!+\!T_3^\star)$.
    \item \textbf{High-quality, ramp-up} (duration $T_4^\star:=T^\star\!-\!T_1^\star\!-\!T_3^\star$): $p^\star\!=\!1$, $b^\star(t)\!=\!B\bigl(\frac{T^\star-t+1}{T_4^\star+1}\bigr)^{1/(2\beta)-1}$ for $t\in[T_1^\star\!+\!T_3^\star,\,T^\star]$.
\end{enumerate}
The phase durations are determined as follows. Phase~1 exhausts all low-quality data:
\[
    T_1^\star = (1-\rho)\frac{D}{B}.
\]
Phase~3 (the ramp-up) has duration
\[
    T_4^\star
    =
    \left(\frac{\eta\sigma_{\good}^2}{sB^{s+2}}\right)^{\!1/\gamma}
    D^{\frac{s+1}{\gamma}}(1+o(1)).
\]
Phase~2 occupies the remainder: $T_3^\star = T^\star - T_1^\star - T_4^\star$, where the total horizon satisfies
\[
    T^\star = \frac{D}{B} - (2\beta-1)T_4^\star + o(T_4^\star).
\]
The optimal excess risk satisfies
\[
    \E[E(\bar{\mb{\theta}}_{T^\star})]
    \asympconst
    B^s D^{-s}.
\]
\end{theorem}

The proof is in Appendix~\ref{app:proof-joint-hard}; the critical case $s=1-1/\beta$, which lies outside the regime dichotomy above, is discussed separately in Remark~\ref{rem:critical-joint-schedule}. We verify Theorem~\ref{thm:joint-optimal} by comparing four strategies that share the same total data budget $D$ and high-quality fraction $\rho$:
\begin{itemize}[nosep]
    \item \textbf{Constant $b$ + uniform $p$}: $b(t)=D/T$, $p(t)=\rho$ (naive baseline).
    \item \textbf{Constant $b$ + late $p$}: $b(t)=D/T$, bang-bang quality placed late.
    \item \textbf{Uniform $p$ + optimal $b$}: $p(t)=\rho$, $b(t)=\max(C\sqrt{K(T\!-\!t)},B)$.
    \item \textbf{Joint optimal}: simultaneously optimal $(b,p)$ from Theorem~\ref{thm:joint-optimal}.
\end{itemize}
Figure~\ref{fig:joint-optimal} confirms that the joint optimal strategy
achieves the lowest final risk in both regimes. In the noise-limited regime, the joint schedule outperforms batch-only optimization by exploiting the batch-size drop at the quality transition; in the signal-limited regime, it outperforms the other three strategies through the three-phase structure. Full experimental details are given in Appendix~\ref{app:joint-experiment}.

\begin{figure*}[ht]
    \centering
    \includegraphics[width=\textwidth]{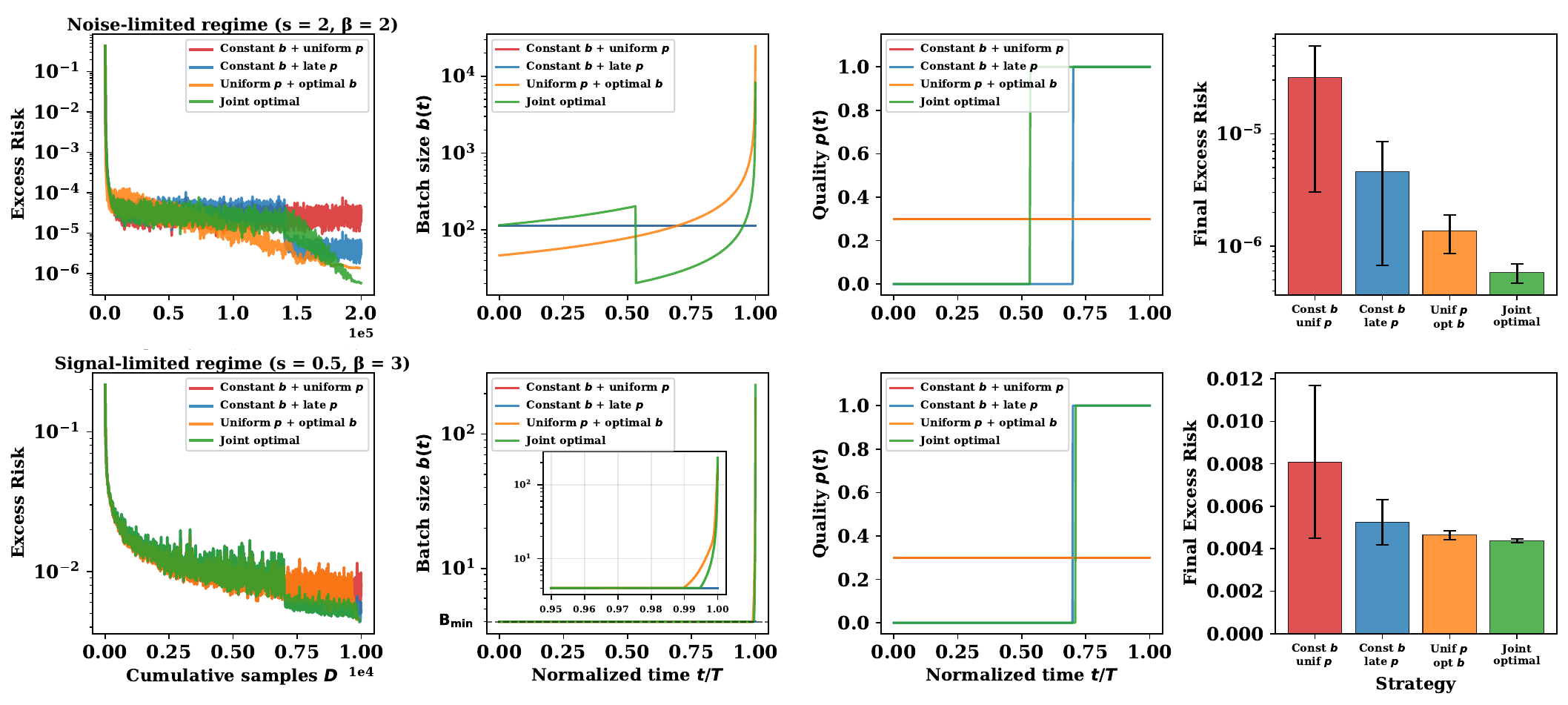}
    \caption{Verification of Theorem~\ref{thm:joint-optimal}.
    Top: noise-limited ($s=2$, $\beta=2$); bottom: signal-limited ($s=0.5$, $\beta=3$).
    Four strategies share the same total data budget $D$ and high-quality fraction $\rho$:
    (1)~constant $b$ + uniform $p$ (baseline),
    (2)~constant $b$ + late $p$,
    (3)~uniform $p$ + optimal $b$,
    (4)~joint optimal (Theorem~\ref{thm:joint-optimal}).
    Columns show loss dynamics, batch-size schedules, data-quality schedules, and final excess risk.
    The joint optimal strategy achieves the lowest final risk in both regimes.}
    \label{fig:joint-optimal}
\end{figure*}

\begin{figure}[!t]
    \centering
    \includegraphics[width=\columnwidth]{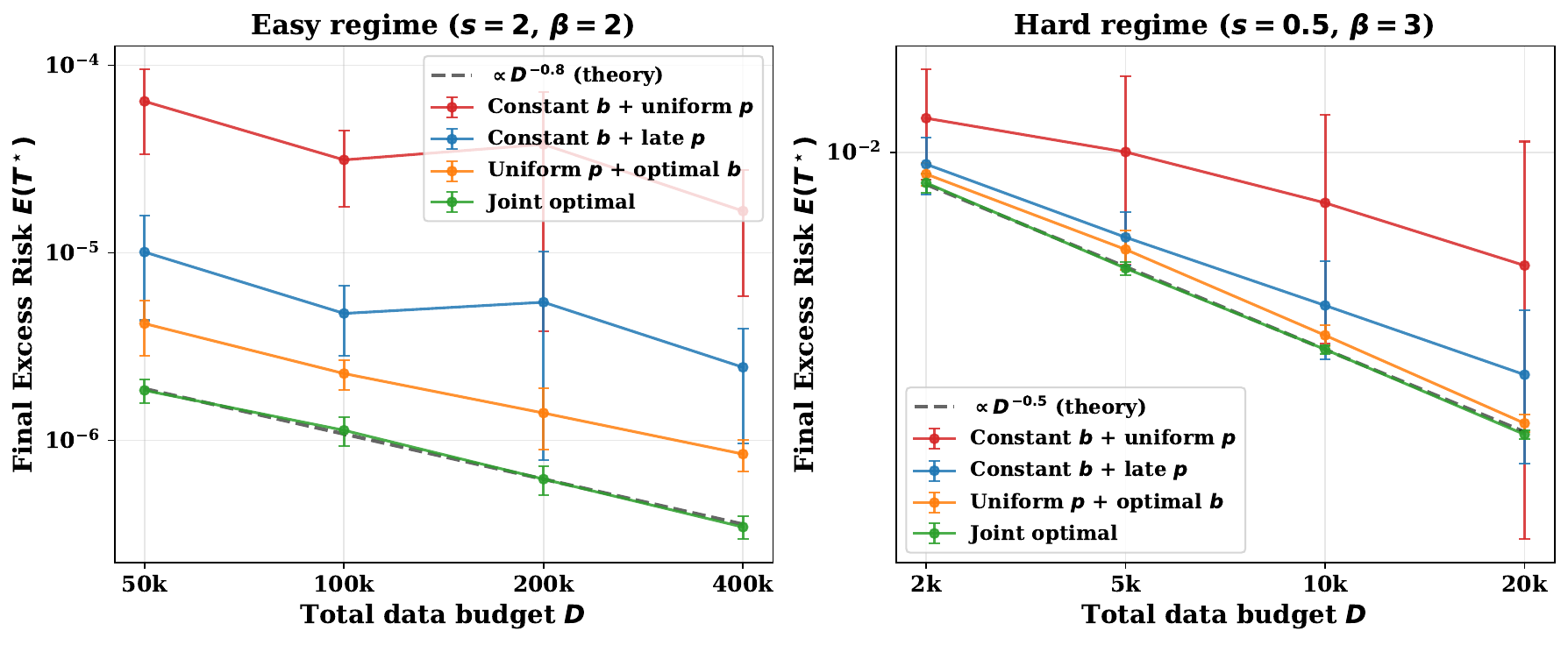}
    \caption{Scaling-law verification of Theorem~\ref{thm:joint-optimal}.
    Final excess risk vs.\ total data budget $D$ (log-log scale).
    Left: noise-limited regime ($s\!=\!2$, $\beta\!=\!2$, theory: $D^{-0.8}$);
    right: signal-limited regime ($s\!=\!0.5$, $\beta\!=\!3$, theory: $D^{-0.5}$).
    The joint optimal strategy matches the predicted power-law scaling in both regimes.}
    \label{fig:scaling-law}
\end{figure}

\paragraph{Scaling-law verification.}
Theorem~\ref{thm:joint-optimal} predicts that the optimal excess risk scales as
$\E[E(\bar{\mb{\theta}}_{T^\star})]\asympconst D^{-s\beta/(1+s\beta)}$ in the noise-limited regime and
$\E[E(\bar{\mb{\theta}}_{T^\star})]\asympconst B^s D^{-s}$ in the signal-limited regime.
To verify this, we sweep the total data budget $D$ and measure the final excess risk achieved by each strategy.
The signal-limited regime uses a smaller data range because its signal-constrained dynamics require longer per-sample training horizons ($T\propto D/B$), making large budgets computationally expensive.
Figure~\ref{fig:scaling-law} shows that the joint optimal strategy closely tracks the predicted power-law exponent in both regimes
($-0.81$ vs.\ theory $-0.80$ for noise-limited; $-0.51$ vs.\ theory $-0.50$ for signal-limited),
while suboptimal strategies exhibit shallower scaling.

\subsection{Interpreting the joint optimum}
\label{sec:interpretation}

The key difference from original FSL is that the quality-aware FSL
introduces data quality as an additional control variable. Instead of treating
the noise level as fixed, the schedule can decide when to allocate
high-quality, low-noise data through $p(t)$. This extra degree of freedom can
be used in two dual ways: it can be converted into more optimization steps at a
comparable noise scale, or into a lower noise scale at a comparable number of
optimization steps.

\paragraph{Noise-limited regime ($s>1-1/\beta$): quality as more signal.}
In this regime, the noise-accumulation integral is the bottleneck. The
batch-only optimum~\citep{wang2025fastcatchup} controls noise by ramping up the batch size, but it treats
all data as having the same noise level. Once high-quality data is available,
the joint optimum lowers the batch size during the low-noise phase. In
particular, when entering the high-quality region $S$, the optimal batch size
drops by a factor $r=\sigma_{\good}^2/\sigma_{\bad}^2$, so that
\[
    \sigma_{\good}^2/b^\star|_S
    =
    \sigma_{\bad}^2/b^\star|_{S^c}.
\]
Thus, the local noise-injection rate remains comparable across quality phases,
but the smaller batch size yields more optimization steps. In the
noise-limited regime, high-quality data is therefore used to obtain more signal
at roughly the same noise scale.

\paragraph{Signal-limited regime ($s<1-1/\beta$): quality as lower noise.}
In this regime, the signal term $t^{-s}$ is the bottleneck, so the optimal
schedule must preserve as many optimization steps as possible. Both the
batch-only and joint optima therefore spend most of training at the minimum
batch size $B_{\min}$. Since the update count is already maximized, the joint
optimum uses high-quality data in a different way: it places low-noise data
late in training. This curriculum placement lowers the noise scale of
late-stage updates without changing the leading-order signal-accumulation
mechanism. In the signal-limited regime, high-quality data is therefore used to
reduce noise at comparable signal accumulation.

\paragraph{Unifying principle.}
The two regimes show the same signal-noise trade-off from opposite sides. When noise is the bottleneck, high-quality data is converted into more signal without increasing noise. When signal is the bottleneck, high-quality data is converted into lower noise without sacrificing signal. As we will detail in Section~\ref{sec:exp-llm}, this duality motivates Drop-Stable-Rampup: drop the batch size after the quality transition to exploit low-noise data for more updates, keep it low to accumulate signal, and ramp up at the end to suppress terminal noise.

% ======================================================================

\section{From Theory to Practice: LLM Experiments}
\label{sec:exp-llm}

\subsection{The Drop-Stable-Rampup recipe}
Building on Section~\ref{sec:interpretation}, we translate the qualitative structure of the joint optimum into a practical schedule for the pretraining-to-midtraining quality transition. A single LLM training run must serve capabilities that may fall into different empirical regimes: general knowledge and fluency often behave as noise-limited, whereas mathematical reasoning and some code-generation tasks exhibit signal-limited trajectories. Since midtraining uses a much smaller curated corpus than pretraining, low-batch training is feasible on existing hardware. This allows us to go beyond conventional curriculum learning with learning rate decay---which exploits high-quality data only for noise reduction. We additionally accumulate signal through increased update intensity. This leads to \textbf{Drop-Stable-Rampup}:
\begin{enumerate}[nosep]
    \item \textbf{Drop}: cut the batch size at the quality transition---accumulating signal via more gradient steps under low-noise conditions.
    \item \textbf{Stable}: hold $B_{\min}$ for a period---maximizing signal acquisition for signal-limited capabilities.
    \item \textbf{Rampup}: linearly grow the batch size (equivalently decay the learning rate~\citep{smith2018dont})---suppressing accumulated noise for final convergence.
\end{enumerate}

\paragraph{Setup.}
Our large-scale proprietary-data evaluation uses a Mixture-of-Experts (MoE) model with 15B total parameters and 2.4B active parameters, adopting the shortcut-connected MoE architecture of~\citet{cai2025shortcutconnected}. All runs in this setting are initialized from a checkpoint pretrained on 1.2T tokens and use a fixed midtraining budget of 108B tokens. We additionally conduct a controlled public-data validation with two architectures---a 600M LLaMA-style dense model and a 15B-total/2.4B-active MoE model. Both public-data models are pretrained on the same 300B FineWeb-Edu corpus and midtrained on the same 30B-token public math/code mixture. Within each architecture, all schedules are token-budget matched and start from the same pretrained checkpoint. Detailed configurations are provided in Appendix~\ref{app:llm-details}. We structure the evaluation into four parts: (1)~testing the batch-size drop across data corpora, including a Constant~2k control, (2)~selecting the stable-phase ratio, (3)~comparing DSR against conventional schedules at large scale, and (4)~validating transfer across architectures while holding public training data fixed.

\subsection{Verifying the effectiveness of dropping batch size}
\label{sec:exp_drop_data_quality}
Theorem~\ref{thm:joint-optimal}(I) predicts that the batch size should drop when training transitions to lower-noise data. To test whether this structural prediction transfers to LLMs, we compare three batch-size schedules: (1)~Constant~4k, (2)~Ramp-up 4k$\to$8k, and (3)~Drop 1k$\to$8k, which drops the pretraining batch size from 4k to 1k at the transition and subsequently ramps it to 8k. We evaluate all three schedules on both the pretraining and midtraining corpora (Appendix~\ref{app:data-composition}).

\begin{table*}[ht]
\centering
\caption{
Unified 15B MoE ablation results. 
Constant~2k is evaluated only on the midtraining corpus.
\textbf{Bold} indicates the best result within each comparison block.
}
\label{tab:main_results}
\scriptsize
\setlength{\tabcolsep}{3.5pt}
\renewcommand{\arraystretch}{1.08}
\resizebox{\textwidth}{!}{%
\begin{tabular}{llccc cccc ccc}
\toprule
\multirow{2}{*}{\textbf{Category}} 
& \multirow{2}{*}{\textbf{Benchmark}}
& \multicolumn{3}{c}{\textbf{Pretraining Corpus}}
& \multicolumn{4}{c}{\textbf{Midtraining Corpus}}
& \multicolumn{3}{c}{\textbf{Midtraining Schedule}} \\
\cmidrule(lr){3-5} 
\cmidrule(lr){6-9}
\cmidrule(lr){10-12}
&
& 4k 
& 4k$\to$8k 
& 1k$\to$8k
& 4k 
& 4k$\to$8k 
& 1k$\to$8k
& 2k
& Cos 
& WSD 
& DSR \\
\midrule

\multirow{4}{*}{Knowledge}
& MMLU      & 54.29 & \textbf{56.05} & 54.92 & 64.90 & 65.35 & \textbf{66.40} & 64.70 & 65.88 & 66.58 & \textbf{67.27} \\
& MMLU-Pro  & 23.61 & 24.21 & \textbf{26.29} & 38.11 & 37.46 & \textbf{39.07} & 35.71 & 40.04 & 40.96 & \textbf{41.89} \\
& CMMLU     & 59.25 & \textbf{60.57} & 60.00 & 66.02 & 66.07 & \textbf{67.01} & 65.04 & 67.43 & 67.61 & \textbf{68.59} \\
& C-Eval    & 57.89 & 59.96 & \textbf{59.99} & 65.61 & 66.71 & \textbf{67.37} & 65.47 & 67.84 & 68.09 & \textbf{68.24} \\
\midrule

\multirow{2}{*}{Reasoning}
& ARC-c     & 67.50 & \textbf{69.50} & 67.00 & 74.50 & 76.00 & \textbf{77.50} & 77.00 & 78.50 & 80.50 & \textbf{81.00} \\
& BBH       & 36.70 & 38.03 & \textbf{38.34} & \textbf{51.68} & 50.30 & 49.60 & 50.09 & 48.13 & 52.19 & \textbf{53.92} \\
\midrule

\multirow{2}{*}{Math}
& GSM8K     & 27.56 & \textbf{29.72} & 26.64 & 61.74 & 61.89 & \textbf{66.74} & 61.28 & 62.97 & 65.59 & \textbf{69.82} \\
& MATH      &  9.20 &  9.38 & \textbf{9.74} & 30.48 & 32.22 & \textbf{36.08} & 33.54 & 33.98 & 36.86 & \textbf{39.66} \\
\midrule

\multirow{6}{*}{Code}
& HE+       & 20.12 & \textbf{21.95} & \textbf{21.95} & 41.46 & \textbf{45.12} & 43.90 & 42.07 & 43.29 & 46.95 & \textbf{49.39} \\
& BCB       & 21.23 & \textbf{23.95} & 22.46 & 35.09 & 37.02 & \textbf{39.04} & 34.74 & \textbf{40.00} & 35.79 & 38.60 \\
& LCB       &  3.30 &  3.30 & \textbf{3.74} & 11.01 & 12.33 & \textbf{12.78} & 10.79 &  9.47 & 12.33 & \textbf{13.88} \\
& CRUX-I    & \textbf{30.38} & 28.12 & 28.75 & 36.12 & 37.38 & \textbf{38.50} & 36.38 & 36.88 & 38.25 & \textbf{39.75} \\
& CRUX-O    & 28.12 & \textbf{28.62} & 27.25 & 35.62 & 37.50 & \textbf{39.75} & 37.38 & 40.38 & 40.00 & \textbf{40.88} \\
& MultiPL-E & 17.26 & \textbf{19.74} & 18.34 & 41.11 & 42.54 & \textbf{43.30} & 43.03 & 43.95 & 44.87 & \textbf{47.53} \\
\midrule

\multicolumn{2}{l}{\textbf{Overall}}
& 32.60 & \textbf{33.79} & 33.24
& 46.68 & 47.71 & \textbf{49.07} & 46.94
& 48.48 & 49.76 & \textbf{51.46} \\

\bottomrule
\end{tabular}%
}
\end{table*}

\begin{figure}[ht]
    \centering
    \includegraphics[width=\linewidth]{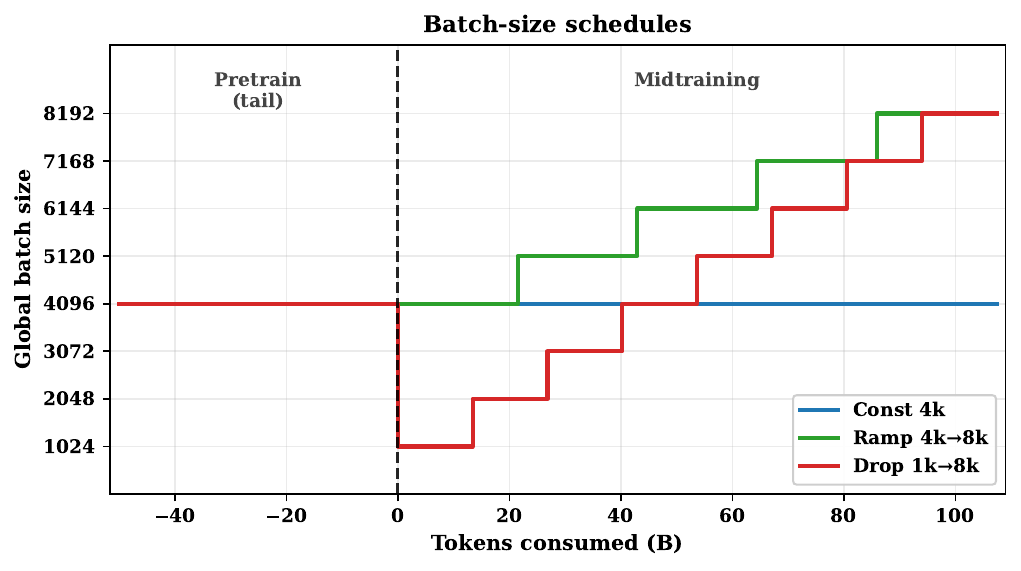}
    \includegraphics[width=\linewidth]{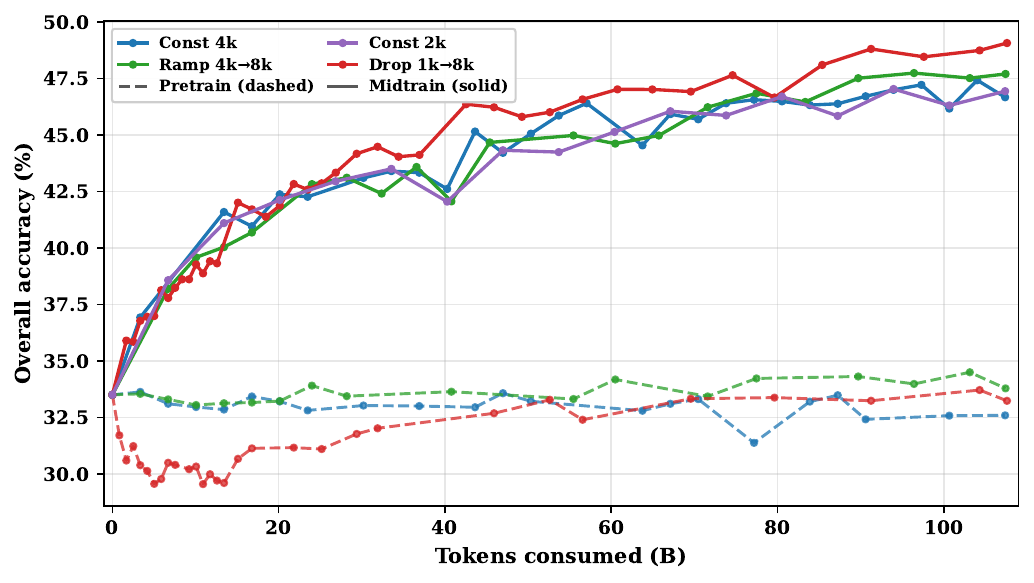}
    \caption{Batch-size schedules and overall-accuracy trajectories on pretraining (dashed) and midtraining (solid) corpora.}
    \label{fig:pre_mid_train}
\end{figure}

As shown in Table~\ref{tab:main_results} and Figure~\ref{fig:pre_mid_train}, the results exhibit a sharp dichotomy that aligns with our theoretical predictions. On the pretraining corpus, the monotonic Ramp-up achieves the best overall performance (33.79), with Drop trailing slightly at 33.24. This is consistent with our theory: on pretraining data, the noise level $\sigma_{\eff}^2$ remains uniformly high, so the signal-accumulation benefit of smaller batch sizes cannot compensate for the increased noise injection. Without a quality transition, the Ramp-up remains preferable.

However, on the high-quality midtraining corpus---where the transition from pretraining to curated data creates precisely the noise drop our theory addresses---the trend reverses. Drop~1k$\to$8k achieves the best overall score (49.07), outperforming the Ramp-up baseline by +1.36. The gains are broad-based, with notable improvements on GSM8K~(+4.85), MATH~(+3.86), and CRUX-O~(+2.25), consistent with the prediction that dropping the batch size at a quality transition yields more gradient steps for signal accumulation under low-noise conditions.

Finally, to rule out the possibility that the improvement arises merely from taking more optimizer updates, we add a Constant~2k control in the midtraining setting. Constant~2k uses more optimizer updates than every other schedule, yet it does not outperform either Ramp~4k$\to$8k or Drop~1k$\to$8k and is not best on any benchmark. Thus, the drop--ramp gain cannot be explained by a smaller constant batch or by the total number of optimizer updates alone.

\subsection{Selecting the stable-phase ratio}
\label{sec:exp_stable_ratio}

\begin{figure}[ht]
    \centering
    \begin{flushleft}
        \includegraphics[width=0.88\linewidth]{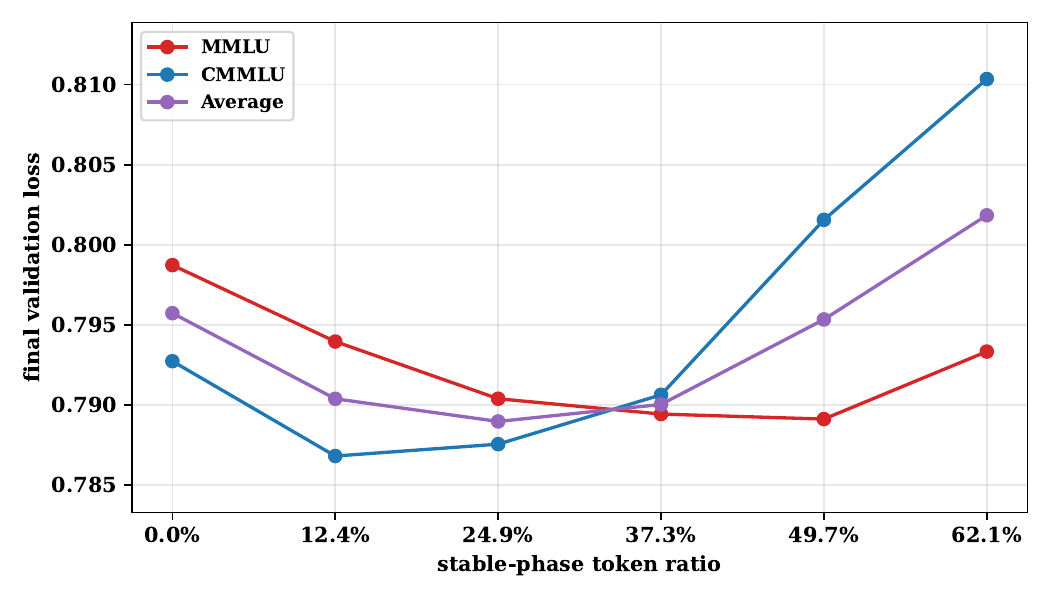}
    \end{flushleft}
    \includegraphics[width=\linewidth]{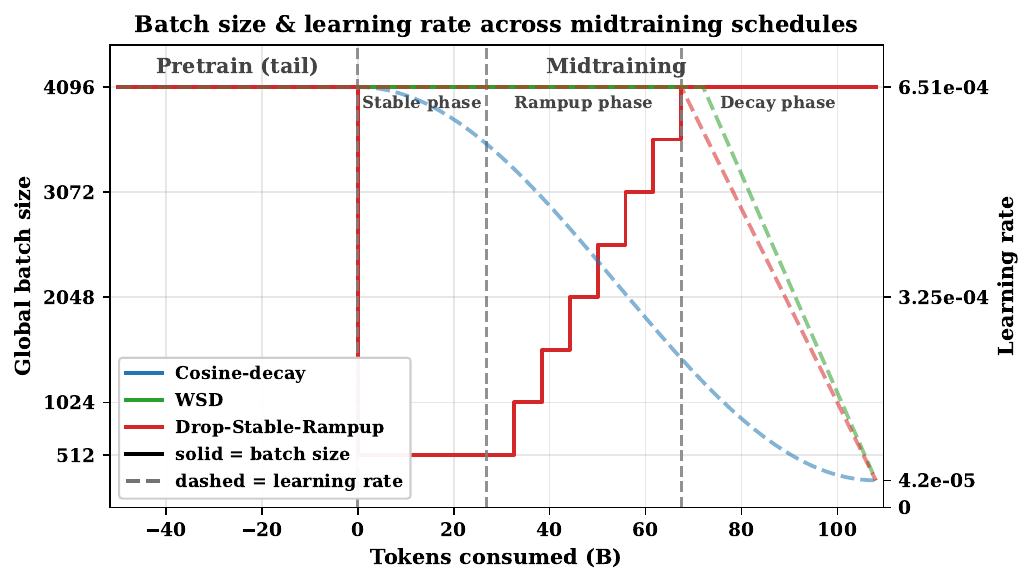}
    \caption{The final validation loss vs.\ stable-phase ratio (top) and schedule trajectories (bottom).}
    \label{fig:stable_ratio}
\end{figure}

Theorem~\ref{thm:joint-optimal}(II) predicts the existence of a flat low-batch-size phase before the terminal ramp-up, with an asymptotic duration that depends on problem-dependent constants ($s$, $\beta$, $\sigma_{\good}^2$) not directly observable in LLM training. We therefore tune the stable-phase ratio empirically, scaling the batch-size bounds to [512, 4096] and varying the token ratio allocated to the stable phase (0\%--62\%) under the fixed 108B budget. The schedule has three phases: stable (batch size fixed at 512), ramp-up (linear increase 512$\to$4096), and decay (learning rate linear decay equivalent to batch size linear growth~\citep{smith2018dont}). We fix equal token allocation for the ramp-up and decay phases and sweep only the stable ratio.

We select the stable ratio using validation loss on MMLU and CMMLU. As shown in Figure~\ref{fig:stable_ratio}, validation loss exhibits a U-shaped curve and reaches its minimum at a stable ratio of 24.9\% in this setting. A too-short stable phase skips the signal-accumulation phase; a too-long stable phase leaves insufficient budget for the terminal ramp-up needed for noise suppression. The two validation sets also exhibit different preferences: MMLU benefits from a longer stable phase, suggesting stronger signal-accumulation sensitivity, whereas CMMLU degrades under overly long stable phases, suggesting greater reliance on noise suppression. Notably, every stable-ratio configuration in this sweep exceeds both WSD and Cosine-decay baselines on overall accuracy, demonstrating the robustness of the Drop-Stable-Rampup strategy (see Figure~\ref{fig:app-rampup-ckpt} for details).

% Declared before the public-data section so the double-column table can float
% upward while remaining within the main-text page limit.
\begin{table*}[!t]
\centering
\scriptsize
\setlength{\tabcolsep}{3.5pt}
\renewcommand{\arraystretch}{0.96}
\setlength{\abovecaptionskip}{3pt}
\setlength{\belowcaptionskip}{0pt}
\resizebox{\textwidth}{!}{%
\begin{tabular}{lcccccccc}
\toprule
\multirow{2}{*}{\textbf{Benchmark}}
& \multicolumn{4}{c}{\textbf{600M dense model}}
& \multicolumn{4}{c}{\textbf{15B-total/2.4B-active MoE model}} \\
\cmidrule(lr){2-5}\cmidrule(lr){6-9}
& \textbf{Cosine} & \textbf{WSD-Early} & \textbf{WSD-Late} & \textbf{DSR}
& \textbf{Cosine} & \textbf{WSD-Early} & \textbf{WSD-Late} & \textbf{DSR} \\
\midrule
MMLU-STEM           & 24.74 & 24.70 & 25.10 & \textbf{27.35} & 33.52 & 33.19 & 34.71 & \textbf{39.57} \\
MMLU-Pro-STEM       &  8.57 &  9.14 &  9.71 & \textbf{11.50} & 13.86 & 13.57 & 16.00 & \textbf{17.14} \\
ARC-c               & 23.00 & 21.50 & 24.00 & \textbf{25.00} & 32.00 & 34.00 & 36.00 & \textbf{49.50} \\
BBH                 & 30.19 & 30.31 & 30.22 & \textbf{31.11} & 31.16 & 31.28 & 31.06 & \textbf{32.92} \\
GSM8K               &  6.70 &  5.77 &  7.78 & \textbf{16.32} & 27.87 & 27.41 & 32.41 & \textbf{41.19} \\
MATH                &  4.26 &  3.62 &  4.70 & \textbf{5.64}  &  7.46 &  7.88 & 10.68 & \textbf{17.40} \\
HE+                 & 17.07 & 18.29 & 20.12 & \textbf{23.78} & 29.27 & 29.88 & 29.27 & \textbf{37.20} \\
MBPP+               & 11.11 & 13.23 & 12.17 & \textbf{20.11} & 24.87 & 29.37 & 33.86 & \textbf{37.57} \\
BCB                 & 10.09 &  9.04 & 10.18 & \textbf{14.91} & 19.56 & 19.74 & 21.32 & \textbf{26.93} \\
LCB                 &  --   &  --   &  --   & --             &  0.72 &  0.00 &  0.72 & \textbf{2.15} \\
MultiPL-E           &  6.21 &  6.06 &  8.23 & \textbf{9.16}  &  9.94 &  9.63 & 11.80 & \textbf{13.98} \\
\midrule
\textbf{Overall} & 14.19 & 14.17 & 15.22 & \textbf{18.49} & 20.93 & 21.45 & 23.44 & \textbf{28.69} \\
\bottomrule
\end{tabular}%
}
\caption{Public-data validation across architectures. Dashes denote unavailable results excluded from the dense-model mean; \textbf{bold} indicates the best result.}
\label{tab:public_cross_arch}
\end{table*}

\subsection{Main results: schedule comparison}

% \begin{figure}[ht]
%     \centering
%     \includegraphics[width=\linewidth]{three_schedules_bs_lr.pdf}
%     % \includegraphics[width=\linewidth]{overall_curve_three_schedules.pdf}
%     \caption{Schedule trajectories for Cosine-decay, WSD, and our Drop-Stable-Rampup.}
%     \label{fig:three_schedules}
% \end{figure}

We evaluate Drop-Stable-Rampup with the selected 24.9\% stable ratio in the 108B-token proprietary-data setting. We compare it against two widely adopted baselines: Cosine-decay~\citep{loshchilov2017sgdr} and WSD~\citep{hu2024minicpm}. Both baselines use a constant batch size of 4096 and the same token budget (see Figure~\ref{fig:stable_ratio} for schedule trajectories). Results are summarized in Table~\ref{tab:main_results}, with further trajectory details provided in Figure~\ref{fig:app-all-curves}.

Drop-Stable-Rampup achieves the highest overall score of 51.46, outperforming WSD (49.76) and Cosine-decay (48.48). The largest gains appear on mathematical reasoning and algorithmic code generation: GSM8K~(+4.23 over WSD), MATH~(+2.80), HumanEval+~(+2.44), and MultiPL-E~(+2.66). These benchmarks also exhibit the smallest early-phase gaps---DSR stays close to WSD during the stable phase and overtakes it during ramp-up---suggesting that the extended low-batch-size phase accumulates latent capacity that manifests after noise suppression. We provide a detailed per-benchmark analysis in Appendix~\ref{sec:regime-analysis}.

\subsection{Validation on public data across architectures}
\label{sec:public-validation}

We further test whether the three-phase DSR structure transfers across architectures without phase-length tuning. We use two models trained on exactly the same public corpora: (i)~a standard 600M-parameter dense LLaMA-style model~\citep{touvron2023llama2}, and (ii)~the same 15B-total/2.4B-active MoE architecture used in the preceding experiments. Each model is pretrained on 300B FineWeb-Edu tokens~\citep{penedo2024fineweb}. Starting from its corresponding pretrained checkpoint, every schedule then consumes the same 30B-token midtraining mixture: 15B tokens from SwallowCode~\citep{fujii2026rewriting} and 15B tokens from the \texttt{4plus} subset of Nemotron-CC-Math~\citep{mahabadi2026nemotronccmath}. Within each architecture, the four schedules share the checkpoint, data and data order, tokenizer, optimizer, random seed, and evaluation protocol; only the batch-size and learning-rate schedules differ. For both architectures, DSR uses a simple equal-third allocation: 10B tokens for a minimum-batch stable phase, 10B for a batch-size ramp, and 10B for a maximum-batch phase with terminal learning-rate decay. Full details are given in Appendix~\ref{app:public-validation-details}.

As shown in Table~\ref{tab:public_cross_arch}, DSR is best on every reported STEM, mathematics, and code benchmark for both architectures, as well as on the general-reasoning evaluations. On the 600M dense model, it improves the mean from 15.22 for WSD-Late, the strongest baseline, to 18.49 (+3.27). On the 2.4B-active MoE model, it improves the mean from 23.44 to 28.69 (+5.25). These results show that the three-phase structure transfers across dense and MoE architectures even with a fixed equal-third allocation. We do not interpret this allocation as universal.

\section{Conclusion}
We study how LLMs should consume scarce high-quality data. 
By solving the joint data-quality and batch-size scheduling problem in asymptotic closed form, we show that optimal training schedules must jointly coordinate data quality and batch size. The resulting joint optimum reveals two regimes, but a shared principle: low-batch training after a quality transition accumulates signal, while late high-quality placement and terminal batch-size ramp-up suppress noise. This principle leads to Drop-Stable-Rampup, which improves LLM midtraining over standard schedules. More broadly, our results suggest that high-quality data should not merely be added to training; it should change how compute is allocated.
% \newpage
\section*{Limitations}

\paragraph{Theoretical simplification.}
Our analysis is built on feature-space linear regression with Gaussian features and power-law spectral decay. While this modeling framework is standard in the scaling-law literature and has been useful for deriving qualitative predictions about training dynamics, it does not capture key aspects of LLM training: nonlinear feature learning, multi-layer compositional representations, and attention-based context dependence. Furthermore, the two-level data-quality model ($\sigma_{\good}^2$ vs.\ $\sigma_{\bad}^2$) is a stylized abstraction; real training corpora exhibit a continuous quality spectrum, and what constitutes ``quality'' is itself task-dependent. Extending the theory to richer noise models and nonlinear dynamics remains an important direction.

\paragraph{Multi-task heterogeneity.}
Our framework models data quality as a single scalar dimension (low-noise vs.\ high-noise), but LLM training is inherently multi-task: different data domains (math, code, knowledge) correspond to distinct subtasks with potentially different optimal schedules. As our empirical trajectory analysis shows (Appendix~\ref{sec:regime-analysis}), benchmarks differ
substantially in their early-phase gaps and delayed gains, and a single schedule must compromise across these heterogeneous dynamics. A more fine-grained theory that jointly optimizes per-domain data mixtures $\{p_i(t)\}$ together with the batch-size $b(t)$ (or learning rate $\eta(t)$) schedule would better capture this heterogeneity, but remains an open problem.

\paragraph{Scale dependence of optimal hyperparameters.}
Our experiments cover a 600M dense model and a 15B-total/2.4B-active MoE model, but this remains a limited sample of model scales and architectures. As model size grows, the balance between signal accumulation and noise suppression may shift, potentially requiring different drop magnitudes, stable-phase durations, or ramp-up rates. Whether the specific hyperparameter choices (e.g., the batch-size drop factor and stable ratio) transfer beyond the evaluated settings, or how they should be adapted, remains an open question that our current framework does not address.

\paragraph{Experimental scope and overhead.}
We evaluate Drop-Stable-Rampup in a large-scale proprietary-data MoE setting and in a controlled public-data study that holds FineWeb-Edu pretraining and the SwallowCode/Nemotron-CC-Math midtraining mixture fixed across a 600M dense model and a 2.4B-active MoE model. This broadens the evidence across architecture, scale, and data access, but does not establish universal transfer to larger dense models, substantially different data budgets, languages, or domain mixtures. Additionally, the low-batch-size stable phase reduces hardware utilization, incurring a $1.1\text{--}1.2\times$ wall-clock overhead across the evaluated settings. These overheads are specific to their respective infrastructure configurations (Appendix~\ref{app:llm-details}); different parallelism strategies or hardware may yield different ratios. Our comparisons are token-budget matched rather than wall-clock matched, so we do not claim GPU-hour or wall-clock optimality.

\section*{Ethical Considerations}

Our work studies training schedules for LLM midtraining and does not release a new model, dataset, or user-facing system. The main risks are therefore indirect. First, the proposed schedule is evaluated through large-scale training experiments, which incur nontrivial computational cost. As discussed in the Limitations section, the low-batch-size stable phase reduces hardware utilization in our setup and leads to a $1.1\text{--}1.2\times$ wall-clock overhead; the practical and environmental cost of such schedules should be considered when scaling them further.

Second, our experiments use large-scale pretraining and midtraining corpora. Such corpora may inherit residual privacy, copyright, or offensive-content risks common to web-scale training data. For the proprietary corpora used in our experiments, data collection and processing follow the existing data-governance pipeline of the data provider, including automated filtering for sensitive or unsafe content. The corpora are used only for aggregate training experiments and benchmark evaluation; this work does not introduce, redistribute, or publish raw training data, nor does it release trained checkpoints. Since large-scale corpora cannot be manually audited in full, residual data-governance risks may remain.

Third, improving the efficiency of midtraining could lower the cost of capability improvement for both beneficial and potentially harmful downstream applications. We mitigate these risks by restricting the contribution to schedule-level analysis, reporting the relevant training setup and compute-efficiency trade-offs, and not releasing trained models or training corpora.

\bibliography{references}

\clearpage
\appendix
\section{Gradient-Noise Covariance}
\label{app:gradient-noise}

Let
\[
    \mb{e}=\mb{\theta}-\mb{\theta}^\star.
\]
For one sample, the stochastic gradient is
\[
\begin{aligned}
    \mb{g}(\mb{\theta};\mb{x},y)
    &=
    \bigl(\langle \mb{\phi}(\mb{x}),\mb{\theta}\rangle-y\bigr)\mb{\phi}(\mb{x}) \\
    &=
    \bigl(\langle \mb{\phi}(\mb{x}),\mb{e}\rangle-\epsilon\bigr)\mb{\phi}(\mb{x}).
\end{aligned}
\]
Taking expectation over \(\mb{x}\) and \(\epsilon\), and using \(\E[\mb{\phi}(\mb{x})]=\mb{0}\),
we obtain \(\E[\mb{g}(\mb{\theta};\mb{x},y)]=\mb{H}\mb{e}\),
confirming that the stochastic gradient is unbiased.

The covariance of the single-sample gradient noise at time \(t\) is
\[
    \mb{\Sigma}(\mb{\theta})
    =
    \operatorname{Cov}
    \bigl(
        (\langle \mb{\phi}(\mb{x}),\mb{e}\rangle-\epsilon)\mb{\phi}(\mb{x})
    \bigr).
\]
Since the label noise is independent of \(\mb{x}\),
\[
    \mb{\Sigma}(\mb{\theta})
    =
    \operatorname{Cov}
    \bigl(\langle \mb{\phi}(\mb{x}),\mb{e}\rangle\mb{\phi}(\mb{x})\bigr)
    +
    \sigma_{\eff}^2(t)\mb{H}.
\]

\begin{lemma}[Quality-aware anisotropic noise]
\label{lem:quality-anisotropic-noise}
For any \(\mb{\theta}\in\R^N\) and any time \(t\), the single-sample gradient-noise
covariance satisfies
\[
    \mb{v}_j^\top \mb{\Sigma}(\mb{\theta})\mb{v}_j
    \asympconst
    \lambda_j\bigl(E(\mb{\theta})+\sigma_{\eff}^2(t)\bigr),
\]
where \(\mb{v}_j\) is the eigenvector of \(\mb{H}\) with eigenvalue \(\lambda_j\), and the constants are independent of \(j\), \(t\), and \(\mb{\theta}\).
\end{lemma}

\begin{proof}[Proof]
Since \(\mb{\phi}(\mb{x})\sim\mathcal N(\mb{0},\mb{H})\), Wick's
formula gives
\[
\begin{aligned}
    &\E\bigl[\mb{\phi}(\mb{x})\mb{\phi}(\mb{x})^\top \mb{e} \mb{e}^\top \mb{\phi}(\mb{x})\mb{\phi}(\mb{x})^\top\bigr]
    \\
    &\qquad =
    2\mb{H}\mb{e}\mb{e}^\top \mb{H} + (\mb{e}^\top \mb{H} \mb{e})\mb{H}.
\end{aligned}
\]
Therefore,
\[
\begin{aligned}
    \operatorname{Cov}\bigl(\langle \mb{\phi}(\mb{x}),\mb{e}\rangle\mb{\phi}(\mb{x})\bigr)
    =&
    \E\bigl[\mb{\phi}\mb{\phi}^\top \mb{e} \mb{e}^\top \mb{\phi}\mb{\phi}^\top\bigr] \\
    &-\mb{H}\mb{e}\mb{e}^\top \mb{H}
    \\
    =&
    \mb{H}\mb{e}\mb{e}^\top \mb{H}+(\mb{e}^\top \mb{H} \mb{e})\mb{H}.
\end{aligned}
\]
Since \(\mb{e}^\top \mb{H} \mb{e}=2E(\mb{\theta})\), we have
\[
    \mb{\Sigma}(\mb{\theta})
    =
    \mb{H}\mb{e}\mb{e}^\top \mb{H} + 2E(\mb{\theta})\mb{H} + \sigma_{\eff}^2(t)\mb{H}.
\]
The lower bound follows from \(\mb{H}\mb{e}\mb{e}^\top \mb{H}\succeq 0\).  For the upper bound, we have
\[
    \mb{H}\mb{e}\mb{e}^\top \mb{H}
    \preceq
    (\mb{e}^\top \mb{H} \mb{e})\mb{H}
    =
    2E(\mb{\theta})\mb{H}.
\]
Hence
\[
    \mb{\Sigma}(\mb{\theta})
    \preceq
    \bigl(4E(\mb{\theta})+\sigma_{\eff}^2(t)\bigr)\mb{H}.
\]
This proves the lemma.
\end{proof}

\section{Proof of Theorem~\ref{thm:quality-fsl}}
\label{pfofthm:quality-fsl}

\begin{proof}
Let \(\mb{e}_t=\bar{\mb{\theta}}_t-\mb{\theta}^\star\). Without loss of generality, assume zero initialization \(\bar{\mb{\theta}}_0=\mb{0}\),
so that \(\mb{e}_0=-\mb{\theta}^\star\) and \(e_{0,j}^2=|\theta_j^\star|^2\). Since \(\nabla R(\bar{\mb{\theta}}_t)=\mb{H}\mb{e}_t\), the SDE~\eqref{eq:sde}
decouples coordinate-wise: each \(e_{t,j}\) satisfies a scalar
Ornstein-Uhlenbeck equation with decay rate \(\lambda_j\). Integrating via Duhamel's formula yields
\[
\begin{aligned}
    e_{t,j}
    &=
    e^{-\lambda_j t}e_{0,j} \\
    &+
    \int_0^t
    e^{-\lambda_j(t-\tau)}
    \sqrt{\frac{\eta}{b(\tau)}}\,
    \bigl[\mb{\Sigma}(\bar{\mb{\theta}}_\tau)^{1/2}\dd \mb{W}_\tau\bigr]_j .
\end{aligned}
\]
By It\^o isometry,
\[
\begin{aligned}
    \E[e_{t,j}^2]
    &=
    e^{-2\lambda_j t}e_{0,j}^2\\
    &+
    \int_0^t
    e^{-2\lambda_j(t-\tau)}
    \frac{\eta}{b(\tau)}
    \E\bigl[q_j(\tau)\bigr]
    \dd \tau,
    \end{aligned}
\]
where
\[
    q_j(\tau):=\mb{v}_j^\top \mb{\Sigma}(\bar{\mb{\theta}}_\tau)\mb{v}_j .
\]
Therefore
\[
\begin{aligned}
    2\E[E(\bar{\mb{\theta}}_t)]
    &=
    \sum_{j\geq 1}
    \lambda_j e^{-2\lambda_j t}e_{0,j}^2
    \\
    +&
    \sum_{j\geq 1}
    \lambda_j
    \int_0^t
    e^{-2\lambda_j(t-\tau)}
    \frac{\eta}{b(\tau)}
    \E[q_j(\tau)]
    \dd \tau .
\end{aligned}
\]
By Lemma~\ref{lem:quality-anisotropic-noise},
\[
    \E[q_j(\tau)]
    \asympconst
    \lambda_j
    \bigl(\E[E(\bar{\mb{\theta}}_\tau)]+\sigma_{\eff}^2(\tau)\bigr).
\]
Hence, if
\[
    f(t):=\E[E(\bar{\mb{\theta}}_t)],
\]
then \(f\) satisfies the Volterra-type relation
\[
    f(t)
    \asympconst
    e(t)
    +
    \int_0^t
    \frac{\eta\,K(t-\tau)}{b(\tau)}
    \bigl(f(\tau)+\sigma_{\eff}^2(\tau)\bigr)\dd \tau,
\]
where
\[
    e(t):=
    \sum_{j\geq 1}\lambda_j |\theta_j^\star|^2 e^{-2\lambda_j t},
    K(t):=
    \sum_{j\geq 1}\lambda_j^2 e^{-2\lambda_j t}.
\]
Under the power-law assumptions \(\lambda_j\asymp j^{-\beta}\) and
\(|\theta_j^\star|^2\asymp j^{-1}\lambda_j^{s-1}\asymp j^{-(1+(s-1)\beta)}\),
the sums defining \(e(t)\) and \(K(t)\) can be approximated by integrals via
Euler-Maclaurin comparison. Substituting the continuous variable
\(x=\lambda_j t\) and using \(j\asymp \lambda_j^{-1/\beta}\), one obtains
\[
\begin{aligned}
    e(t)
    &=
    \sum_{j\geq 1}\lambda_j |\theta_j^\star|^2 e^{-2\lambda_j t} \\
    & \asympconst
    \int_0^\infty x^{s-1} e^{-2x}\dd x\cdot t^{-s}
    \asympconst t^{-s},
\end{aligned}
\]
and similarly
\[
\begin{aligned}
    K(t)
    &=
    \sum_{j\geq 1}\lambda_j^2 e^{-2\lambda_j t}\\
    &\asympconst
    \int_0^\infty x^{1-1/\beta} e^{-2x}\dd x\cdot t^{-(2-1/\beta)}\\
    &\asympconst (t+1)^{-(2-1/\beta)}.
\end{aligned}
\]

Now define
\[
    g(t)
    :=
    e(t)
    +
    \eta\int_0^t
    K(t-\tau)\frac{\sigma_{\eff}^2(\tau)}{b(\tau)}\dd \tau,
\]
and define the linear Volterra operator
\[
    (\mathcal{T}f)(t)
    :=
    \eta\int_0^t
    K(t-\tau)\frac{f(\tau)}{b(\tau)}\dd \tau.
\]
With this notation, the preceding relation admits the compact representation
\[
    f=g+\mathcal{T}f
\]
up to universal multiplicative constants.  Formally,
\[
    f=(I-\mathcal{T})^{-1}g=\sum_{m=0}^{\infty}\mathcal{T}^m g.
\]
It remains to check that the higher-order terms generated by the multiplicative
part of the gradient noise do not change the leading scaling.  Since
\(K(t)\asympconst (t+1)^{-(2-1/\beta)}\) is nonincreasing and integrable for \(\beta>1\),
\[
    (K*K)(t)
    =
    \int_0^t K(t-\tau)K(\tau)\dd \tau
    \lesssim K(t).
\]
Using \(b(t)\geq B_{\min}\geq 1\), we obtain
\[
    \mathcal{T}^2g(t)
    \lesssim
    \eta\,\mathcal{T}g(t).
\]
Iterating this estimate and taking \(\eta\) small enough yields
\[
    f(t)
    \asympconst
    g(t)+\mathcal{T}g(t).
\]
Moreover, because \(\sigma_{\good}^2\leq \sigma_{\eff}^2(t)\leq
\sigma_{\bad}^2\), the term \(\mathcal{T}g(t)
\) is bounded by the same two leading quantities contained in \(g(t)\).  Indeed,
\[
\begin{aligned}
    \mathcal{T}e(t)
    &\lesssim
    \eta\int_0^t K(t-\tau)\frac{1}{b(\tau)}\dd \tau \\
    &\lesssim
    \eta\int_0^t K(t-\tau)
    \frac{\sigma_{\eff}^2(\tau)}{b(\tau)}\dd \tau,
    \end{aligned}
\]
and
\[
\begin{aligned}
    &\mathcal{T}\left[
        \eta\int_0^{\cdot}
        K(\cdot-r)\frac{\sigma_{\eff}^2(r)}{b(r)}\dd r
    \right](t)\\
    &\lesssim
    \eta
    \int_0^t
    K(t-r)
    \frac{\eta\sigma_{\eff}^2(r)}{b(r)}
    \dd r .
\end{aligned}
\]
Thus \(\mathcal{T}g(t)\lesssim g(t)\), and consequently
\[
    f(t)
    \asympconst
    e(t)
    +
    \eta\int_0^t
    K(t-\tau)\frac{\sigma_{\eff}^2(\tau)}{b(\tau)}\dd \tau.
\]
Substituting \(e(t)\asympconst t^{-s}\) proves the theorem.
\end{proof}

\section{Proof of Theorem~\ref{thm:fixed-b}}
\label{pfofthm:fixed-b}
\begin{proof}
The optimization problem is
\[
    \max_{0\leq p\leq 1}
    \int_0^T
    \frac{K(T-t)}{b(t)}p(t)\dd t
\]
subject to
\[
    \int_0^T p(t)b(t)\dd t=\rho D.
\]
Introduce the measure
\(\dd \mu(t):=b(t)\dd t\).
Then the constraint becomes
\(\int_0^T p(t)\dd \mu(t)=\rho D\),
and the objective becomes
\[
    \int_0^T S_b(t)p(t)\dd \mu(t).
\]
This is a continuous knapsack problem. The optimizer fills the high-quality
budget on the set where \(S_b(t)\) is largest, giving the stated bang-bang
structure.
\end{proof}

\providecommand{\BD}{\mathcal B_D}

\section{Proof of Theorem~\ref{thm:joint-optimal}}
\label{app:proof-joint-hard}

Throughout this section we write
\(\sigma_1^2:=\sigma_{\good}^2\) and \(\sigma_2^2:=\sigma_{\bad}^2\)
for brevity, and denote the objective of~\eqref{eq:joint_problem} by
\(\mathcal{E}[T,b,p]\).

For data budget \(D>0\), define the feasible set
\[
\begin{aligned}
\BD:=\bigl\{(T,b,p):{}&\, T>0,\; b\in L^1(0,T),\\
&\, p\in L^\infty(0,T),\\
&\, b(t)\ge B,\; 0\le p(t)\le1\ \text{a.e.},\\
&\, \textstyle\int_0^T\! b(t)\dd t=D,\\
&\, \textstyle\int_0^T\! b(t)p(t)\dd t=\rho D\bigr\}.
\end{aligned}
\]
We consider
\[
\mathcal{E}^\star(D):=\min_{(T,b,p)\in\BD}\mathcal{E}[T,b,p],
\]
where
\[
\mathcal{E}[T,b,p]
:=
T^{-s}
+\eta\!\int_0^T
\frac{K(T-t)\,\sigma_{\eff}^2(t)}{b(t)}
\dd t .
\]

For the power kernel \(K(u)=(u+1)^{-\gamma}\) with
\(\gamma:=2-1/\beta>0\), the horizon is bounded by \(T\le D/B\) (from
\(b\ge B\)) and bounded away from zero by the bias term \(T^{-s}\), so the
problem is well posed and the standard pointwise Lagrangian characterization
for decomposable integral programs applies.

\begin{lemma}[Canonical structure of an optimal joint schedule]
\label{lem:canonical-joint-schedule}
Assume
\[
B>0,\quad D>0,\quad 0<\rho<1,
\quad \sigma_1<\sigma_2 .
\]
Write \(w(t):=K(T^\star-t)\). Then an optimal solution can be chosen to have the following canonical form.

\begin{enumerate}
    \item \textbf{Strict-switching case.}
    There exist \(T_s\in[0,T^\star]\) and constants \(C_0,C_1>0\) such that
    \[
    p^\star(t)=
    \begin{cases}
        0, & 0\le t<T_s,\\
        1, & T_s\le t\le T^\star,
    \end{cases}
    \]
    and
    \[
    b^\star(t)=
    \begin{cases}
        \max\bigl\{B,\,C_0\sqrt{w(t)}\bigr\},
        & t<T_s,\\[2pt]
        \max\bigl\{B,\,C_1\sqrt{w(t)}\bigr\},
        & t\ge T_s.
    \end{cases}
    \]

    \item \textbf{Terminal-tie case.}
    There exist \(T_s\in[0,T^\star]\) and \(C>0\) such that
    \(p^\star(t)=0\) for \(t<T_s\), while on
    \(I_s:=[T_s,T^\star]\) both labels are pointwise optimal and
    \(p^\star(t)\in\{0,1\}\) may be chosen measurably subject to
    \[
    \int_0^{T^\star}b^\star(t)p^\star(t)\dd t=\rho D.
    \]
    The batch schedule is
    \[
    b^\star(t)=
    \begin{cases}
        \max\bigl\{B,\,C\sqrt{w(t)}\bigr\}, & t<T_s,\\[2pt]
        r^{p^\star(t)}\,C\sqrt{w(t)}, & t\in I_s,
    \end{cases}
    \]
    where \(r=\sigma_1^2/\sigma_2^2\), with \(rC\sqrt{w(t)}\ge B\) on \(I_s\).
\end{enumerate}
\end{lemma}

\begin{proof}
Fix an optimal solution and write it as
\[
(T^\star,b^\star,p^\star).
\]
Since \((T^\star)^{-s}\) is a constant, it suffices to minimize the noise integral
over \((b,p)\) at this fixed horizon. Recall \(w(t)=K(T^\star-t)\),
which is nondecreasing in \(t\).

Introduce Lagrange multipliers \(\mu,\nu\) for the data and quality
constraints. Since the integrand is pointwise separable in \(t\), the
optimizer satisfies the following pointwise condition for a.e.\ \(t\):
\[
\begin{aligned}
&(b^\star(t),p^\star(t))
\in
\mathop{\arg\min}_{\substack{b\ge B\\0\le p\le1}}
\biggl\{
\frac{w(t)}{b}
\bigl(\sigma_2^2\\
&\qquad{}-(\sigma_2^2-\sigma_1^2)p\bigr)
+\mu b+\nu bp
\biggr\}.
\end{aligned}
\]

Let
\[
\Delta:=\sigma_2^2-\sigma_1^2>0.
\]
For \(i=0,1\), define
\[
a_0:=\mu,\qquad c_0:=\sigma_2^2,
\]
and
\[
a_1:=\mu+\nu,\qquad c_1:=\sigma_1^2.
\]
For a fixed \(t, b\), the two endpoint label costs are
\[
L_0(t,b):=\frac{c_0w(t)}{b}+a_0b,
\]
and
\[
L_1(t,b):=\frac{c_1w(t)}{b}+a_1b.
\]
The pointwise Lagrangian can be written as
\[
(1-p)L_0(t,b)+pL_1(t,b).
\]

Define the reduced endpoint costs
\[
\Phi_i(w)
:=
\min_{b\ge B}
\left\{
\frac{c_iw}{b}+a_ib
\right\},
\qquad i=0,1.
\]
Then pointwise optimality gives
\[
p^\star(t)=0
\quad\text{if}\quad
\Phi_0(w(t))<\Phi_1(w(t)),
\]
and
\[
p^\star(t)=1
\quad\text{if}\quad
\Phi_1(w(t))<\Phi_0(w(t)).
\]
Indeed, if \(\Phi_0(w)<\Phi_1(w)\), then for every \(b\ge B\) and
\(p>0\),
\[
(1-p)L_0(t,b)+pL_1(t,b)
>
\Phi_0(w),
\]
while \(p=0\) and a minimizer of \(L_0\) achieve \(\Phi_0(w)\). The case
\(\Phi_1(w)<\Phi_0(w)\) is symmetric. Thus non-endpoint choices of \(p\) can
only occur on tie sets where
\[
\Phi_0(w(t))=\Phi_1(w(t)).
\]
On such tie sets, possible fractional values of \(p\) may be removed without
changing feasibility or the objective. Indeed, let
\[
F:=\{t:0<p^\star(t)<1\}.
\]
On \(F\), the first-order condition in \(p\) gives
\[
\Delta\frac{w(t)}{b^\star(t)}=\nu b^\star(t).
\]
Since the measure \(b^\star(t)\dd t\) is nonatomic, there exists a measurable set
\(S\subseteq F\) such that
\[
\int_S b^\star(t)\dd t=\int_F b^\star(t)p^\star(t)\dd t.
\]
Define
\[
\tilde p(t)=
\begin{cases}
1, & t\in S,\\
0, & t\in F\setminus S,\\
p^\star(t), & t\notin F.
\end{cases}
\]
Then
\[
\int_0^{T^\star}b^\star(t)\tilde p(t)\dd t
=
\int_0^{T^\star}b^\star(t)p^\star(t)\dd t.
\]
Moreover, the objective change on \(F\) is
\[
\begin{aligned}
&-\Delta\!\int_F\!\frac{w(t)}{b^\star(t)}
(\tilde p(t)-p^\star(t))\dd t\\
&\quad=
-\nu\!\int_F\! b^\star(t)(\tilde p(t)-p^\star(t))\dd t
=0.
\end{aligned}
\]
Hence we may choose an optimal solution with
\[
p^\star(t)\in\{0,1\}
\quad\text{a.e.}
\]

We next show that
\[
\nu>0.
\]
Suppose \(\nu\le0\). For any \(t\) with \(w(t)>0\) and any \(b\ge B\),
\[
L_1(t,b)-L_0(t,b)
=
-\frac{\Delta w(t)}{b}+\nu b
<0.
\]
Hence \(p^\star(t)=1\) for a.e.\ \(t\in[0,T^\star]\), contradicting
\[
\int_0^{T^\star}b^\star(t)(1-p^\star(t))\dd t=(1-\rho)D>0.
\]
Therefore
\[
\nu>0.
\]
Consequently,
\[
a_1>a_0,
\qquad
c_1<c_0.
\]
Moreover, the endpoint reduced costs are finite and attained only if
\(a_0>0\) and \(a_1>0\).

We now prove that the \(p=1\) region is terminal (i.e., corresponds to
late training times) and identify the two possible canonical structures.
A direct calculation gives
\[
\Phi_i(w)
=
\begin{cases}
a_iB+\dfrac{c_iw}{B},
& 0<w\le w_i,\\[6pt]
2\sqrt{a_ic_iw},
& w\ge w_i,
\end{cases}
\]
where \(w_i:=a_iB^2/c_i\).
Since \(a_1>a_0\) and \(c_1<c_0\), we have \(w_0<w_1\).

\medskip
\noindent\textbf{Claim.}\;
The ratio \(R(w):=\Phi_1(w)/\Phi_0(w)\) is nonincreasing on
\((0,\infty)\), with
\[
\lim_{w\downarrow0}R(w)=\frac{a_1}{a_0}>1,
\quad
\lim_{w\to\infty}R(w)=\sqrt{\frac{a_1c_1}{a_0c_0}}.
\]

\begin{proof}[Proof of Claim]
We verify the monotonicity on the three regions determined by
\(0<w_0<w_1\).

\emph{Region \(0<w\le w_0\)} (both constrained):
\[
R(w)=\frac{a_1B+c_1w/B}{a_0B+c_0w/B}.
\]
Differentiating, the numerator of \(R'(w)\) is proportional to
\(a_0c_1-a_1c_0<0\). Hence \(R\) is strictly decreasing.

\emph{Region \(w_0\le w\le w_1\)} (\(\Phi_0\) unconstrained, \(\Phi_1\)
constrained):
\[
\begin{aligned}
R(w)
&=
\frac{a_1B+c_1w/B}{2\sqrt{a_0c_0w}}\\
&=
\frac{a_1B}{2\sqrt{a_0c_0}}\,w^{-1/2}
+
\frac{c_1}{2B\sqrt{a_0c_0}}\,w^{1/2}.
\end{aligned}
\]
Differentiating, the sign of \(R'(w)\) is that of \(c_1w/B-a_1B\),
which is negative for \(w<w_1\). Hence \(R\) is strictly decreasing.

\emph{Region \(w\ge w_1\)} (both unconstrained):
\[
R(w)
=
\frac{2\sqrt{a_1c_1w}}{2\sqrt{a_0c_0w}}
=
\sqrt{\frac{a_1c_1}{a_0c_0}},
\]
which is constant.
Continuity at \(w_0\) and \(w_1\) is verified by direct substitution.
\end{proof}

Since \(R\) is nonincreasing with \(R(0^+)>1\), the set
\(\{w:R(w)\le1\}\) is either empty or a terminal half-line. Because
\(w(t)=K(T^\star\!-\!t)\) is nondecreasing in \(t\), the preimage
\(\{t:p^\star(t)=1\}\) is a terminal interval in time (up to the threshold
level set, on which the split may be chosen measurably). The structure
depends on \(\lim_{w\to\infty}R(w)=\sqrt{a_1c_1/(a_0c_0)}\).

\paragraph{Case 1: \(a_1c_1>a_0c_0\).}
Then \(R(w)>1\) for all \(w>0\), so pointwise optimality forces
\(p^\star(t)=0\) a.e., contradicting
\(\int_0^{T^\star}b^\star(t)p^\star(t)\dd t=\rho D>0\).
This case is excluded.

\paragraph{Case 2: \(a_1c_1<a_0c_0\) (strict-switching).}
Then \(R(\infty)<1\). From the proof of claim, the equation \(R(w)=1\) has a unique solution. Hence there
exists \(T_s\in[0,T^\star]\) such that
\[
p^\star(t)
=
\begin{cases}
0, & 0\le t<T_s,\\
1, & T_s\le t\le T^\star.
\end{cases}
\]
On each region, the optimal batch size is obtained by minimizing
\(L_i(t,b)=c_iw(t)/b+a_ib\) over \(b\ge B\), giving
\(b^\star(t)=\max\{B,\,\sqrt{c_i/a_i}\sqrt{w(t)}\}\). Writing
\[
\begin{aligned}
C_0&:=\sqrt{\tfrac{c_0}{a_0}}=\sqrt{\tfrac{\sigma_2^2}{\mu}},\\
C_1&:=\sqrt{\tfrac{c_1}{a_1}}=\sqrt{\tfrac{\sigma_1^2}{\mu+\nu}},
\end{aligned}
\]
we obtain
\[
b^\star(t)
=
\begin{cases}
\max\{B,\,C_0\sqrt{w(t)}\},
& 0\le t<T_s,\\[2pt]
\max\{B,\,C_1\sqrt{w(t)}\},
& T_s\le t\le T^\star.
\end{cases}
\]

\paragraph{Case 3: \(a_1c_1=a_0c_0\) (terminal-tie).}
Then \(R(w)\ge1\) for all \(w\), with \(R(w)=1\) exactly on
\([w_1,\infty)\) (since \(R\) is strictly decreasing on \((0,w_1)\) and
constant on \([w_1,\infty)\)). The corresponding tie interval in time is
\(I_s:=[T_s,T^\star]\). Before \(I_s\), we have \(p^\star(t)=0\). On \(I_s\),
both labels are pointwise optimal, and \(p^\star(t)\in\{0,1\}\) may be
chosen measurably subject to
\(\int_0^{T^\star}b^\star(t)p^\star(t)\dd t=\rho D\).

From \(a_1c_1=a_0c_0\), i.e., \((\mu+\nu)\sigma_1^2=\mu\sigma_2^2\),
we obtain \(C_1=rC_0\) where \(r=\sigma_1^2/\sigma_2^2\in(0,1)\).
Writing \(C:=C_0\):
\[
b^\star(t)
\!=\!
\begin{cases}
\max\bigl\{B,\,C\!\sqrt{w(t)}\bigr\},
& t<T_s,\\[2pt]
C\sqrt{w(t)},
& t\!\in\! I_s,\; p^\star\!=\!0,\\[2pt]
rC\sqrt{w(t)},
& t\!\in\! I_s,\; p^\star\!=\!1.
\end{cases}
\]
Since \(I_s\) lies in the unconstrained region \(w\ge w_1\), we have
\(rC\sqrt{w(t)}\ge B\) on \(I_s\).

\medskip
After excluding Case~1, only the strict-switching and terminal-tie
canonical forms remain. This proves the lemma.
\end{proof}

\begin{proof}[Proof of Theorem~\ref{thm:joint-optimal}]
We prove the noise-limited regime (\(s>1-1/\beta\)) and the signal-limited
regime (\(s<1-1/\beta\)) separately.
Throughout, write \(\delta:=1/(2\beta)\).

\paragraph{Step 1: the unconstrained problem.}
First ignore the lower-bound constraint \(b(t)\geq B\). Fix \(T>0\). By Lemma~\ref{lem:canonical-joint-schedule}, an optimizer can be chosen with
\[
p(t)\in\{0,1\}
\]
a.e. Let
\[
S:=\{t:p(t)=1\},
\quad
S^c:=\{t:p(t)=0\}.
\]
For fixed \(S\), Cauchy-Schwarz gives
\[
\int_S\frac{K(T-t)}{b(t)}\dd t
\geq
\frac{\left(\int_S \sqrt{K(T\!-\!t)}\dd t\right)^2}
{\int_S b(t)\dd t},
\]
and similarly on \(S^c\). Equality holds iff \(b(t)\propto \sqrt{K(T\!-\!t)}\)
on each region.

Let
\[
\begin{aligned}
A_1&:=\int_S \!\sqrt{K(T\!-\!t)}\dd t,\\
A_0&:=\int_{S^c}\!\sqrt{K(T\!-\!t)}\dd t,
\end{aligned}
\]
and \(A(T):=A_0+A_1=\int_0^T\!\sqrt{K(T\!-\!t)}\dd t\).
The data constraints give
\[
\int_S b(t)\dd t=\rho D,
\qquad
\int_{S^c}b(t)\dd t=(1-\rho)D.
\]
Thus, for fixed \(T\), the unconstrained variance term
\[
\sigma_2^2\!\int_{S^c}\!\frac{K(T\!-\!t)}{b(t)}\dd t
\;+\;
\sigma_1^2\!\int_S\!\frac{K(T\!-\!t)}{b(t)}\dd t
\]
is at least
\[
\frac{\sigma_2^2A_0^2}{(1-\rho)D}
+
\frac{\sigma_1^2A_1^2}{\rho D}.
\]
Minimizing over \(A_0+A_1=A(T)\) gives
\begin{align*}
A_0^\star
&=
\frac{(1-\rho)\sigma_1^2}{(1-\rho)\sigma_1^2+\rho\sigma_2^2}\,A(T),\\[4pt]
A_1^\star
&=
\frac{\rho\sigma_2^2}{(1-\rho)\sigma_1^2+\rho\sigma_2^2}\,A(T).
\end{align*}
The minimized variance term is
\[
\frac{\eta\,\kappa\, A(T)^2}{D},
\qquad
\kappa
=
\frac{\sigma_1^2\sigma_2^2}{(1-\rho)\sigma_1^2+\rho\sigma_2^2}.
\]
Hence the unconstrained problem reduces to
\[
\min_{T>0}
\left\{
T^{-s}+\frac{\eta\,\kappa\, A(T)^2}{D}
\right\}.
\]

The first-order condition is
\[
-sT^{-s-1}
+
\frac{2\eta\,\kappa}{D}A(T)A'(T)=0.
\]
Since
\[
A'(T)=(T+1)^{\delta-1}
\]
and
\[
A(T)=\frac{T^\delta}{\delta}(1+o(1)),
\]
we obtain
\[
T_{\rm unc}
=
\left(
\frac{sD}{4\beta\,\eta\,\kappa}
\right)^{\frac{1}{s+1/\beta}}
(1+o(1)).
\]
At this optimum,
\[
\frac{\eta\,\kappa\, A(T_{\rm unc})^2}{D}
=
s\beta T_{\rm unc}^{-s}(1+o(1)).
\]
Therefore
\[
\mathcal{E}_{\rm unc}^\star(D)
=
(1+s\beta)T_{\rm unc}^{-s}(1+o(1)).
\]

\paragraph{Step 2: the noise-limited regime.}
We first record the explicit form of the unconstrained optimal batch
profile. From the Cauchy-Schwarz step in Step~1, equality requires
\(b(t)\propto\sqrt{K(T-t)}\) on each of the two regions \(S\) and \(S^c\).
Hence the unconstrained optimizer at horizon \(T_{\rm unc}\) takes the
piecewise form
\[
b^\star(t)
=
\begin{cases}
C_0\sqrt{K(T_{\rm unc}\!-\!t)}, & t\in S^c\ (p^\star=0),\\[2pt]
C_1\sqrt{K(T_{\rm unc}\!-\!t)}, & t\in S\ (p^\star=1),
\end{cases}
\]
where \(C_0,C_1\) are the same constants as in
Lemma~\ref{lem:canonical-joint-schedule}, evaluated at \(T=T_{\rm unc}\). They
are determined by the data constraints
\(\int_{S^c}b^\star\!=\!(1\!-\!\rho)D\) and \(\int_S b^\star\!=\!\rho D\):
\[
C_0=\frac{(1-\rho)D}{A_0^\star},
\qquad
C_1=\frac{\rho D}{A_1^\star}.
\]
Substituting the optimal split \((A_0^\star,A_1^\star)\) found in Step~1 yields
\[
\begin{aligned}
C_0=\frac{(1\!-\!\rho)\sigma_1^2\!+\!\rho\sigma_2^2}{\sigma_1^2}\cdot
\frac{D}{A(T_{\rm unc})},
\\
C_1=\frac{(1\!-\!\rho)\sigma_1^2\!+\!\rho\sigma_2^2}{\sigma_2^2}\cdot
\frac{D}{A(T_{\rm unc})}.
\end{aligned}
\]
Since \(\sigma_1<\sigma_2\), we have \(C_1<C_0\); since \(K\) is decreasing,
\(\sqrt{K(T_{\rm unc}\!-\!t)}\) is minimized at \(t=0\), and consequently so is
\(b^\star(t)\). Hence
\[
\min_{t\in[0,T_{\rm unc}]}b^\star(t)
=
C_1(T_{\rm unc}+1)^{\delta-1}.
\]
Using \(A(T_{\rm unc})\asympconst T_{\rm unc}^\delta/\delta\) together with
the explicit form of \(C_1\), this gives
\[
\min_t b^\star(t)
\asympconst
\frac{D}{T_{\rm unc}}.
\]

If \(s>1-1/\beta\), then \(s+1/\beta>1\) and so \(T_{\rm unc}=o(D)\) by
Step~1. Therefore
\[
\min_t b^\star(t)\to\infty
\quad\text{as}\quad
D\to\infty.
\]
For any fixed \(B\), the constraint \(b(t)\ge B\) is therefore inactive for
all sufficiently large \(D\), and the unconstrained solution is also the
constrained optimum.

The ratio
\(C_1/C_0=\sigma_1^2/\sigma_2^2=r\) places this regime in the
\emph{terminal-tie case} in Lemma~\ref{lem:canonical-joint-schedule}, with \(T_s=0\) and tie interval
\(I_s=[0,T^\star]\): the entire horizon lies in the unconstrained branch, and
\(S\subseteq[0,T^\star]\) may be chosen as any measurable set satisfying
\(\int_S b^\star(t)\dd t=\rho D\). This proves the noise-limited regime.

\paragraph{Step 3: the signal-limited regime and the strict-switching class.}
Now assume
\[
\beta>1,
\qquad
0<s<1-\frac1\beta.
\]
Recall \(\gamma:=2-1/\beta\), so that \(K(u)=(u+1)^{-\gamma}\).
By Step~1, the unconstrained optimizer satisfies
\[
T_{\rm unc}
\asympconst
\left(\frac{sD}{4\beta\,\eta\,\kappa}\right)^{\frac{1}{s+1/\beta}}.
\]
Since \(s+1/\beta<1\), we have
\[
\frac{T_{\rm unc}}{D}\to\infty,
\]
i.e., \(T_{\rm unc}\) grows faster than linearly in \(D\).
On the other hand, every feasible schedule satisfies
\[
D=\int_0^T b(t)\dd t\geq BT,
\]
so
\[
T\leq T_{\max}:=\frac{D}{B}.
\]
Thus the unconstrained optimizer is infeasible, and the lower-bound constraint
is active.

We first solve the problem asymptotically within the strict-switching
canonical family. This does not assume that this family is globally optimal;
it only computes the best comparison value among strict-switching schedules.

A strict-switching schedule has, in chronological order, up to four
possible phases:
\[
\begin{array}{l|cccc}
\text{phase} & \text{1} & \text{2} & \text{3} & \text{4}\\\hline
(p,b) & (0,B) & (0,{>}B) & (1,B) & (1,{>}B)\\
\text{duration} & T_1 & T_2 & T_3 & T_4
\end{array}
\]
That is, \(T_i\) denotes the length of phase \(i\), so that
\(T=T_1+T_2+T_3+T_4\). The \(p=0\) flat (resp.\ unconstrained) phase has
length \(T_1\) (resp.\ \(T_2\)); the \(p=1\) flat (resp.\ unconstrained)
phase has length \(T_3\) (resp.\ \(T_4\)). Any of the \(T_i\) may be zero.

\paragraph{Step 3a: optimization of the \(T_2=0\) strict-switching family.}
We first solve the strict-switching subfamily with \(T_2=0\), so that only
phases 1, 3, 4 are present. The \(p=0\) data constraint forces
\[
T_1=\frac{(1-\rho)D}{B}=(1-\rho)T_{\max}.
\]

Parametrize phase 4 by terminal age \(u:=T-t\in[0,T_4]\). On phase 4 take
\(p(t)=1\) and
\[
b(T\!-\!u)
=
B\!\left(
\frac{u+1}{T_4+1}
\right)^{\!\delta-1},
\quad 0\le u\le T_4.
\]
Since \(\delta<1\), this satisfies \(b(T-u)\ge B\), and it matches the lower
bound at \(u=T_4\).

The amount of \(p=1\) data consumed in phase 4 is
\[
\begin{aligned}
D_4
&=
\int_0^{T_4}
B\!\left(
\frac{u+1}{T_4+1}
\right)^{\!\delta-1}
\dd u\\
&=
\frac{B}{\delta}
\bigl[
(T_4+1)-(T_4+1)^{1-\delta}
\bigr].
\end{aligned}
\]
The remaining \(p=1\) data is carried by phase 3, so
\[
\rho D=BT_3+D_4,
\]
and hence
\[
T_3
=
\frac{\rho D}{B}
-
\frac{(T_4+1)-(T_4+1)^{1-\delta}}{\delta}.
\]
For the optimizer found below, \(T_4=o(D)\), so \(T_3>0\) for all sufficiently
large \(D\).

The total horizon is \(T=T_1+T_3+T_4\), and therefore
\[
T
=
\frac{D}{B}
-
\left(\frac1\delta-1\right)T_4
+
O(T_4^{\,1-\delta}).
\]
Consequently, in the regime \(1\ll T_4\ll D\),
\[
\begin{aligned}
T^{-s}
&=
B^sD^{-s}
+
\alpha_D T_4\\
&\quad+
O(D^{-s-1}T_4^{\,1-\delta})
+
O(D^{-s-2}T_4^{\,2}),
\end{aligned}
\]
where
\[
\alpha_D
:=
s\left(\frac1\delta-1\right)
B^{s+1}D^{-s-1}.
\]

We next compute the leading variance correction. Phase 3 corresponds to
terminal ages \(u\in[T_4,T_4+T_3]\). Its contribution is
\[
V_3(T_4)
=
\frac{\eta\,\sigma_1^2}{B}
\int_{T_4}^{T_4+T_3}
(u+1)^{-\gamma}
\dd u.
\]
Since \(T_3=O(D)\) and \(T_4=o(D)\),
\[
V_3(T_4)
=
\frac{\eta\,\sigma_1^2}{B(\gamma-1)}
T_4^{\,-(\gamma-1)}(1+o(1)).
\]
Phase 4 contributes
\[
V_4(T_4)
=
\int_0^{T_4}
\frac{\eta\,\sigma_1^2(u+1)^{-\gamma}}
{B\bigl(\tfrac{u+1}{T_4+1}\bigr)^{\delta-1}}
\dd u.
\]
Using
\[
\gamma=2-\frac1\beta=2-2\delta,
\]
we obtain
\[
\begin{aligned}
V_4(T_4)
&=
\frac{\eta\,\sigma_1^2}{B}
(T_4\!+\!1)^{\delta-1}
\!\int_0^{T_4}\!
(u\!+\!1)^{\delta-1}
\dd u\\
&=
\frac{\eta\,\sigma_1^2}{B\delta}\,
T_4^{\,-(\gamma-1)}(1+o(1)).
\end{aligned}
\]
Thus
\[
V_3(T_4)+V_4(T_4)
=
\xi\,T_4^{\,-(\gamma-1)}(1+o(1)),
\]
where
\[
\xi
:=
\frac{\eta\,\sigma_1^2}{B}
\left(
\frac1{\gamma-1}+\frac1\delta
\right).
\]

Hence the leading correction to \(B^sD^{-s}\) in the \(T_2=0\)
strict-switching family is
\[
\alpha_D T_4+\xi T_4^{\,-(\gamma-1)}.
\]
Minimizing this expression gives
\[
\alpha_D=(\gamma-1)\,\xi T_4^{\,-\gamma}.
\]
Since
\[
(\gamma-1)\!\left(\frac1{\gamma-1}+\frac1\delta\right)
=
\frac1\delta-1,
\]
we obtain the optimal phase-4 length within the \(T_2=0\)
strict-switching family:
\[
T_4
=
\left(
\frac{\eta\,\sigma_1^2}{sB^{s+2}}
\right)^{\!1/\gamma}
D^{\frac{s+1}{\gamma}}.
\]
Since \(s<\gamma-1\), \(T_4=o(D)\), confirming the assumption \(T_4\ll D\)
used above. This family achieves
\[
\mathcal{E}_{T_2=0}^\star(D)
=
B^sD^{-s}
+
O\left(
D^{-\frac{(s+1)(\gamma-1)}{\gamma}}
\right).
\]
In particular,
\[
\mathcal{E}_{T_2=0}^\star(D)=B^sD^{-s}+o(D^{-s}).
\]

\paragraph{Step 3b: elimination of \(T_2\) and conclusion for the strict-switching class.}
We now show that no asymptotically optimal strict-switching schedule can have
a nontrivial \(p=0\) unconstrained phase.

Let \((T,b,p)\) be any strict-switching schedule. Since \(b(t)\ge B\), every
feasible schedule satisfies
\(T\le T_{\max}\). Moreover, the variance term is nonnegative, and hence
\[
\mathcal{E}[T,b,p]\ge T^{-s}\ge T_{\max}^{-s}=B^sD^{-s}.
\]
On the other hand, Step 3a constructs a \(T_2=0\) strict-switching schedule
with value
\[
B^sD^{-s}+o(D^{-s}).
\]
Therefore, any strict-switching schedule that can be asymptotically optimal
must satisfy
\[
T=T_{\max}(1+o(1)).
\]
Indeed, if \(T\le (1-c)T_{\max}\) along a subsequence for some \(c>0\), then
\[
T^{-s}\ge (1-c)^{-s}B^sD^{-s},
\]
whose leading constant is strictly larger than the value achieved in Step 3a.

Consequently,
\[
T_{\max}-T=o(T_{\max}).
\]
Using
\[
D=\int_0^T b(t)\dd t
=
BT+\int_0^T(b(t)-B)\dd t,
\]
we get
\[
\frac1B
\int_0^T(b(t)-B)\dd t
=
T_{\max}-T
=
o(T_{\max}).
\]
Thus the total excess data above the lower bound is \(o(D)\).

Splitting \(b=B+(b-B)\) on \(\{p=1\}\) and using that the total excess above
the lower bound is \(o(D)\), the \(p=1\) data constraint gives
\[
\begin{aligned}
\rho D
&=\int_{\{p=1\}}\!b(t)\dd t\\
&=B\!\int_{\{p=1\}}\!\dd t+\int_{\{p=1\}}\!(b-B)\dd t\\
&= B\!\int_{\{p=1\}}\!\dd t+o(D),
\end{aligned}
\]
so the total time occupied by the \(p=1\) phases satisfies
\[
\int_{\{p=1\}}\!\dd t\;\ge\;\rho T_{\max}-o(T_{\max}).
\]
Hence the \(p=0\) unconstrained phase, if present, lies at terminal ages
\[
u\ge \rho T_{\max}-o(T_{\max})=O(D).
\]
On this phase,
\[
K(u)=(u+1)^{-\gamma}\le C D^{-\gamma}
\]
for some constant \(C>0\) independent of \(D\).

Suppose the \(p=0\) unconstrained phase has positive excess data. Define
\[
\Delta_2
:=
\frac1B
\int_{T_1}^{T_1+T_2}\!(b(t)-B)\dd t.
\]
Replace this phase by a \(p=0\), \(b=B\) segment carrying the same amount of
\(p=0\) data. This preserves both data constraints and increases the total
horizon by \(\Delta_2\).

The modification takes place only at terminal ages \(O(D)\). Therefore
the variance can increase by at most
\[
C D^{-\gamma}\Delta_2.
\]
On the other hand, since the original horizon satisfies \(T=T_{\max}(1+o(1))\), the
increase of the horizon decreases the bias term by at least
\[
cD^{-s-1}\Delta_2
\]
for some constant \(c>0\) independent of \(D\), and for all sufficiently
large \(D\). In the signal-limited regime,
\(s<1-1/\beta\), and hence
\[
s+1<\gamma.
\]
Therefore
\[
D^{-s-1}\gg D^{-\gamma}.
\]
The bias improvement strictly dominates the possible variance loss. Thus any
positive \(p=0\) unconstrained phase can be removed to improve the objective
for all sufficiently large \(D\). Hence an asymptotically optimal
strict-switching schedule must satisfy
\[
T_2=0.
\]

Therefore the \(T_2=0\) family optimized in Step 3a is asymptotically
exhaustive within the strict-switching class. The corresponding phase
durations and objective value are given by Step 3a; their explicit forms
are recorded in Step 5 below.

\paragraph{Step 4: exclusion of terminal-tie schedules.}
It remains to show that terminal-tie schedules cannot improve on the
strict-switching value.

Consider any feasible terminal-tie schedule. All \(p=1\) data lies in a
terminal tie interval. Let \(L\) be the length of this interval and write
\(u=T-t\in[0,L]\). Whenever \(p=1\) in this interval,
\[
b(T-u)=rC(u+1)^{\delta-1}.
\]
Since the tie interval is contained in the region where the \(p=1\)
unconstrained branch is active,
\[
rC(L+1)^{\delta-1}\geq B.
\]
Therefore, on the \(p=1\) part of the tie interval,
\[
\frac{b(T-u)}{B}
\geq
\left(\frac{u+1}{L+1}\right)^{\delta-1}.
\]

Let
\[
S:=\{u\in[0,L]:p(T-u)=1\},
\quad
m:=|S|.
\]
The \(p=1\) data constraint gives
\[
\rho T_{\max}
=
\int_S \frac{b(T-u)}{B}\dd u.
\]
Hence
\[
\rho T_{\max}
\geq
\int_S
\left(\frac{u+1}{L+1}\right)^{\delta-1}\dd u.
\]
The integrand \(f(u):=\bigl(\tfrac{u+1}{L+1}\bigr)^{\delta-1}\) is
decreasing in \(u\). By the rearrangement inequality, for any measurable
\(S\subseteq[0,L]\) with \(|S|=m\),
\[
\begin{aligned}
\int_S f(u)\dd u
\;&\ge\;
\int_{L-m}^{L} f(u)\dd u \\
&=
\frac{L+1}{\delta}\!\left[
1-\left(1-\frac{m}{L+1}\right)^{\!\delta}
\right].
\end{aligned}
\]
(Since \(f\) is decreasing, replacing \(S\) by the right-most subset of equal
measure can only reduce the integral.) Combining with the previous
inequality gives
\[
\rho T_{\max}
\geq
\frac{L+1}{\delta}\!\left[
1-\left(1-\frac{m}{L+1}\right)^{\!\delta}
\right].
\]
Solving for \(m\),
\[
m
\leq
(L+1)
\left[
1-
\left(
1-\frac{\delta\rho T_{\max}}{L+1}
\right)^{1/\delta}
\right].
\]

We now bound this uniformly over \(L\leq T\leq T_{\max}\). If
\[
L+1\leq \delta\rho T_{\max},
\]
then \(m\leq L+1\leq\delta\rho T_{\max}\), which is already strictly smaller
than \(\rho T_{\max}\) by a linear amount. Otherwise the function
\[
x\mapsto
x\left[
1-\left(1-\frac{\delta\rho T_{\max}}{x}\right)^{\!1/\delta}
\right]
\]
is increasing for \(x>\delta\rho T_{\max}\). Indeed, with
\(z=\delta\rho T_{\max}/x\), it equals
\[
\delta\rho T_{\max}\,\frac{1-(1-z)^{1/\delta}}{z},
\]
and the last ratio is decreasing in \(z\), because
\[
z\mapsto 1-(1-z)^{1/\delta}
\]
is concave and vanishes at \(0\). Hence, since \(L+1\leq T_{\max}+1\),
\[
m
\leq
T_{\max}\!\left[
1-(1-\delta\rho)^{1/\delta}
\right]
+o(T_{\max}).
\]
Define
\[
c_{\rho,\delta}
:=
\rho-
\left[
1-(1-\delta\rho)^{1/\delta}
\right].
\]
Because \(0<\rho<1\) and \(0<\delta<1\),
\[
(1-\delta\rho)^{1/\delta}>1-\rho.
\]
Thus
\[
c_{\rho,\delta}>0,
\]
and
\[
m\leq \rho T_{\max}-c_{\rho,\delta}T_{\max}+o(T_{\max}).
\]

The time spent on \(p=0\) is at most
\[
(1-\rho)T_{\max},
\]
because \(b\geq B\). Hence any terminal-tie schedule satisfies
\[
T
\leq
(1-\rho)T_{\max}
+
\bigl(\rho T_{\max}-c_{\rho,\delta}T_{\max}+o(T_{\max})\bigr),
\]
and therefore
\[
T\leq (1-c_{\rho,\delta})T_{\max}+o(T_{\max}).
\]
Its bias term is therefore bounded below by
\[
T^{-s}
\geq
(1-c_{\rho,\delta})^{-s}
B^sD^{-s}(1+o(1)).
\]
Since \(c_{\rho,\delta}>0\), the leading prefactor of \(B^sD^{-s}\) is strictly
larger than \(1\). Therefore every terminal-tie schedule has objective strictly larger,
at the leading constant level, than
\[
B^sD^{-s}+o(D^{-s}),
\]
which is achieved by the strict-switching family. Thus terminal-tie schedules
are not asymptotically optimal.

\paragraph{Step 5: conclusion for the signal-limited regime.}
By Lemma~\ref{lem:canonical-joint-schedule}, every optimal schedule may be
chosen in either the strict-switching case or the terminal-tie case. Step 4
excludes the terminal-tie case. Hence the global asymptotic optimum is given by
the strict-switching optimum computed in Step 3.

Therefore
\[
T_1^\star
=
\frac{(1-\rho)D}{B},
\]
\[
T_4^\star
=
\left(
\frac{\eta\sigma_1^2}{sB^{s+2}}
\right)^{1/\gamma}
D^{\frac{s+1}{\gamma}}
(1+o(1)),
\]
and
\[
T_3^\star
=
\frac{\rho D}{B}
-
\frac{(T_4^\star+1)-(T_4^\star+1)^{1-\delta}}{\delta}
+
o(T_4^\star).
\]
Moreover,
\[
T^\star
=
\frac{D}{B}
-
\left(\frac1\delta-1\right)T_4^\star
+
O\bigl((T_4^\star)^{1-\delta}\bigr),
\]
or equivalently,
\[
T^\star
=
\frac{D}{B}
-
(2\beta-1)T_4^\star
+
O\left((T_4^\star)^{1-\frac1{2\beta}}\right).
\]

Finally, every feasible schedule satisfies
\[
\mathcal{E}[T,b,p]\geq T^{-s}\geq B^sD^{-s},
\]
because \(T\leq D/B\). Together with the strict-switching upper bound, this
gives
\[
\mathcal{E}^\star(D)
=
B^sD^{-s}
+
O\left(
D^{-\frac{(s+1)(\gamma-1)}{\gamma}}
\right),
\]
which, expanded in \(\beta\), reads
\[
\frac{(s+1)(\gamma-1)}{\gamma}
=
\frac{(s+1)(1-1/\beta)}{2-1/\beta}.
\]
This is the claimed signal-limited regime expansion. The proof is complete.
\end{proof}

\begin{remark}[Critical regime]
\label{rem:critical-joint-schedule}
The proof deliberately excludes the critical case
\[
s=1-\frac1\beta .
\]
At criticality, the subcritical balance gives \(T_4\asympconst c\,D\) for
some constant \(c>0\), so the terminal boundary layer is no longer
sublinear. The expansion used to
derive the signal-limited schedule therefore no longer applies.

The argument forcing the \(p=0\) unconstrained phase to vanish uses
\[
s+1<2-\frac1\beta .
\]
At criticality this becomes equality, so the bias loss from shortening the
training horizon and the variance gain from extra batching are of the same
order. Similarly, the argument excluding the terminal-tie form is no longer
decisive, because different canonical forms may differ only at the level of
leading constants.

Nevertheless, Lemma~\ref{lem:canonical-joint-schedule} still applies. Thus,
if an optimal solution exists, it may be chosen in one of the canonical forms
described there. Determining which canonical form is selected, and which phase
proportions vanish, requires solving the corresponding finite-dimensional
constant-level problem.
\end{remark}

\section{Experiments Details}

\subsection{Experimental details for Theorem~\ref{thm:quality-fsl} verification}
\label{app:fsl-experiment}

\begin{figure*}[ht]
    \centering
    \includegraphics[width=\textwidth]{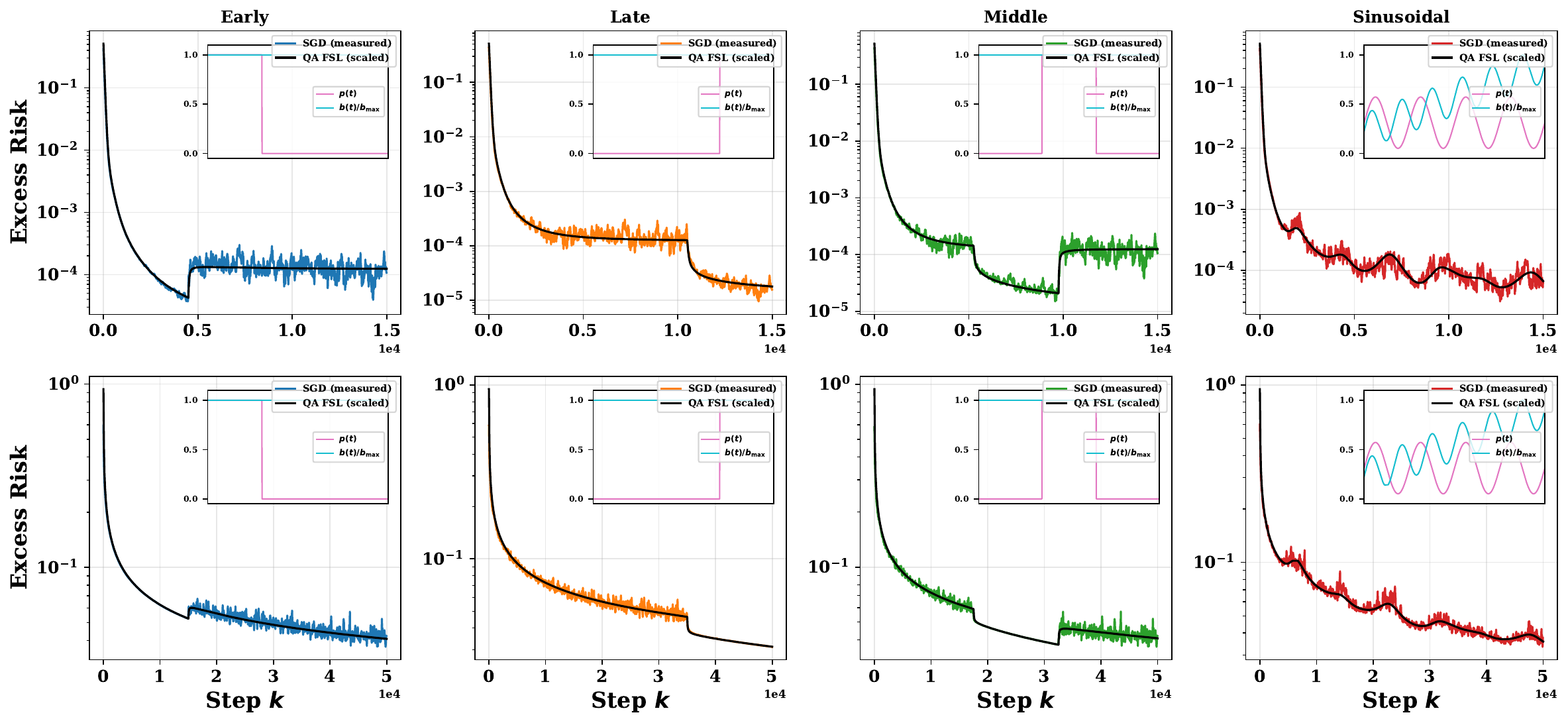}
    \caption{Verification of Theorem~\ref{thm:quality-fsl}.
    Top: noise-limited ($s=2$). Bottom: signal-limited ($s=0.4$).
    Colored curves: measured SGD risk (averaged over 10 seeds). Black curves: quality-aware FSL theory (scaled).}
    \label{fig:fsl-verification}
\end{figure*}

We verify Theorem~\ref{thm:quality-fsl} by running one-pass SGD in the
diagonal feature space, where the assumptions hold exactly. We test both
the noise-limited regime and the signal-limited regime with four schedules satisfying the same total high-quality data budget
\(\sum_{k=1}^K p_k B_k = \rho \sum_{k=1}^K B_k\),
where \(\rho\in(0,1)\) is the fraction of high-quality data:
\begin{itemize}[nosep]
    \item \textbf{Early / Late / Middle}: bang-bang quality placement
    (\(p=1\) on the first, last, or central \(\rho K\) steps) with
    constant batch.
    \item \textbf{Sinusoidal}: jointly varying quality
    \(p(t)\propto\rho+0.25\sin(8\pi t/T)\) and oscillating-upward batch
    \(b(t)\propto 1+3t/T+0.8\sin(12\pi t/T)\), both rescaled to satisfy
    the budget constraints.
\end{itemize}
Figure~\ref{fig:fsl-verification} shows the results. In each panel, the
measured excess risk (colored, averaged over 10 seeds) closely tracks the scaled quality-aware FSL
prediction (black) throughout training. Inset plots display the schedules
\(p(t)\) and \(b(t)\). The theory accurately tracks all
schedule-dependent differences---including the complex Sinusoidal case
where both controls vary simultaneously---confirming that the forgetting
kernel \(K(t-\tau)\) correctly weights \(\sigma_{\eff}^2(\tau)/b(\tau)\) at each \(\tau\).

\paragraph{Parameter settings.}
Throughout, we fix dimension \(N=500\), capacity exponent \(\beta=2\),
learning rate \(\eta=0.01\), and high-quality data fraction \(\rho=0.3\),
with critical exponent \(s_{\mathrm{crit}}=1-1/\beta=0.5\).

\emph{Noise-limited setting} (\(s=2\)): total iterations \(K=15{,}000\)
(\(T=150\)), reference batch size \(B=32\), noise variances
\(\sigma_{\good}^2=0.1\), \(\sigma_{\bad}^2=1.0\).

\emph{Signal-limited setting} (\(s=0.4\)): total iterations \(K=50{,}000\)
(\(T=500\)), reference batch size \(B=4\), noise variances
\(\sigma_{\good}^2=0.1\), \(\sigma_{\bad}^2=10.0\).  In the signal-limited
regime, the bias term \(t^{-s}\) decays slowly and dominates unless the
noise strength is sufficiently large. We therefore use a smaller batch size
and a larger bad-noise variance to ensure that the noise integral is comparable to
the bias \(T^{-s}\), so that the effect of different quality schedules is
visible.

\paragraph{Schedule construction.}
The Early, Late, and Middle schedules use a constant batch size \(b(t)=B\)
and bang-bang quality schedules that place \(p(t)=1\) on the first, last, or
central \(\rho K\) steps respectively.

The Sinusoidal schedule uses a joint non-constant design:
\[
\begin{aligned}
    &b_{\mathrm{raw}}(t)
    =
    1 + 3\frac{t}{T} + 0.8\sin\!\Bigl(\frac{2\pi\cdot 6\, t}{T}\Bigr),
    \\
    &p_{\mathrm{raw}}(t)
    =
    \rho + 0.25\sin\!\Bigl(\frac{2\pi\cdot 4\, t}{T}\Bigr).
\end{aligned}
\]
The batch schedule is rescaled so that
\(\sum_{k=1}^K B_k = BK\) (matching the total data budget of the
constant-batch schedules), and the quality schedule is rescaled so that
\(\sum_{k=1}^K p_k B_k = \rho\sum_{k=1}^K B_k\), with values clipped to
\([0,1]\). This produces an oscillating-upward batch curve (6 cycles) and a
sinusoidal quality curve (4 cycles) that jointly satisfy the budget
constraints.

\paragraph{Theoretical prediction curve.}
For the theoretical prediction, we compute the kernel
\(K(t)=\sum_{j=1}^N\lambda_j^2 e^{-2\lambda_j t}\) and the bias term
\(e(t)=\sum_{j=1}^N\lambda_j|\theta_j^\star|^2 e^{-2\lambda_j t}\) exactly
in the finite-dimensional setting (rather than using the asymptotic
power-law approximation). The noise integral
\(\int_0^t K(t-\tau)\sigma_{\eff}^2(\tau)/b(\tau)\dd\tau\) is discretized
with step size \(\eta\). Since \(b(\tau)\) and \(p(\tau)\) now vary across
schedules, the integrand \(\sigma_{\eff}^2(\tau)/b(\tau)\) is evaluated
pointwise at each discrete time step.

\paragraph{Fitting procedure.}
Since the theorem establishes an asymptotic
equivalence (\(\asympconst\)) rather than an exact equality, we fit a
single global multiplicative constant (via the median of
measured/predicted ratios) to align the theory curve with the measurement.
All experimental curves are averaged over 10 independent random seeds.

\subsection{Experimental details for Corollary~\ref{cor:quality-shape} verification}
\label{app:corollary-experiment}

\begin{figure}[ht]
    \centering
    \includegraphics[width=0.48\textwidth]{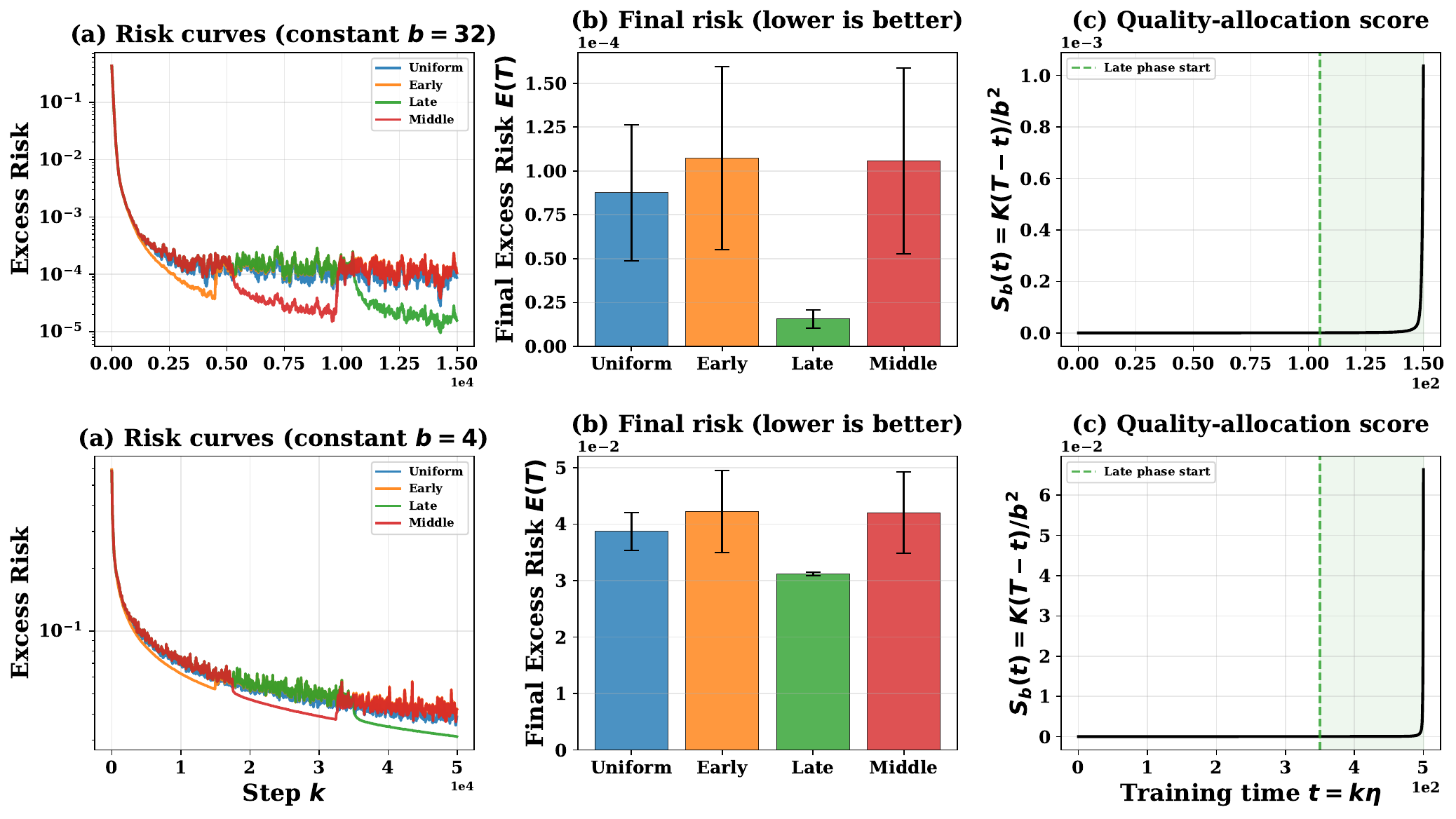}
    \caption{Corollary~\ref{cor:quality-shape}(i): constant batch.
    Top: noise-limited (\(s=2\)); bottom: signal-limited (\(s=0.4\)).
    (a)~Risk curves: Late achieves the lowest final risk.
    (b)~Final excess risk bar chart.
    (c)~Score \(S_b(t)\propto K(T-t)\) is monotonically increasing;
    shaded region marks where Late places high-quality data.}
    \label{fig:constant-batch}
\end{figure}

\paragraph{Constant-batch experiment (Corollary~\ref{cor:quality-shape}(i)).}
We use the same settings as in Appendix~\ref{app:fsl-experiment}:
noise-limited (\(s=2\), \(B=32\), \(\sigma_{\bad}^2=1\), \(K=15{,}000\)) and
signal-limited (\(s=0.4\), \(B=4\), \(\sigma_{\bad}^2=10\), \(K=50{,}000\)),
with \(N=500\), \(\beta=2\), \(\eta=0.01\), \(\sigma_{\good}^2=0.1\), and
\(\rho=0.3\) throughout. All results are averaged over 10 seeds.

When \(b(t)\) is constant, the score reduces to
\(S_b(t)\propto K(T-t)\), which is monotonically increasing in \(t\):
noise injected late is remembered most strongly, so the optimal quality
schedule concentrates high-quality data at the end of training.
Figure~\ref{fig:constant-batch} confirms the prediction in both the noise-limited
and signal-limited regimes. The Late schedule achieves the lowest final risk,
with a consistent ranking Late \(<\) Uniform \(<\) Middle \(<\) Early.
Panel~(c) shows that \(S_b(t)\) is monotonically increasing, and the Late
schedule places its high-quality data precisely in the region where
\(S_b(t)\) is largest.

\begin{figure}[ht]
    \centering
    \includegraphics[width=0.48\textwidth]{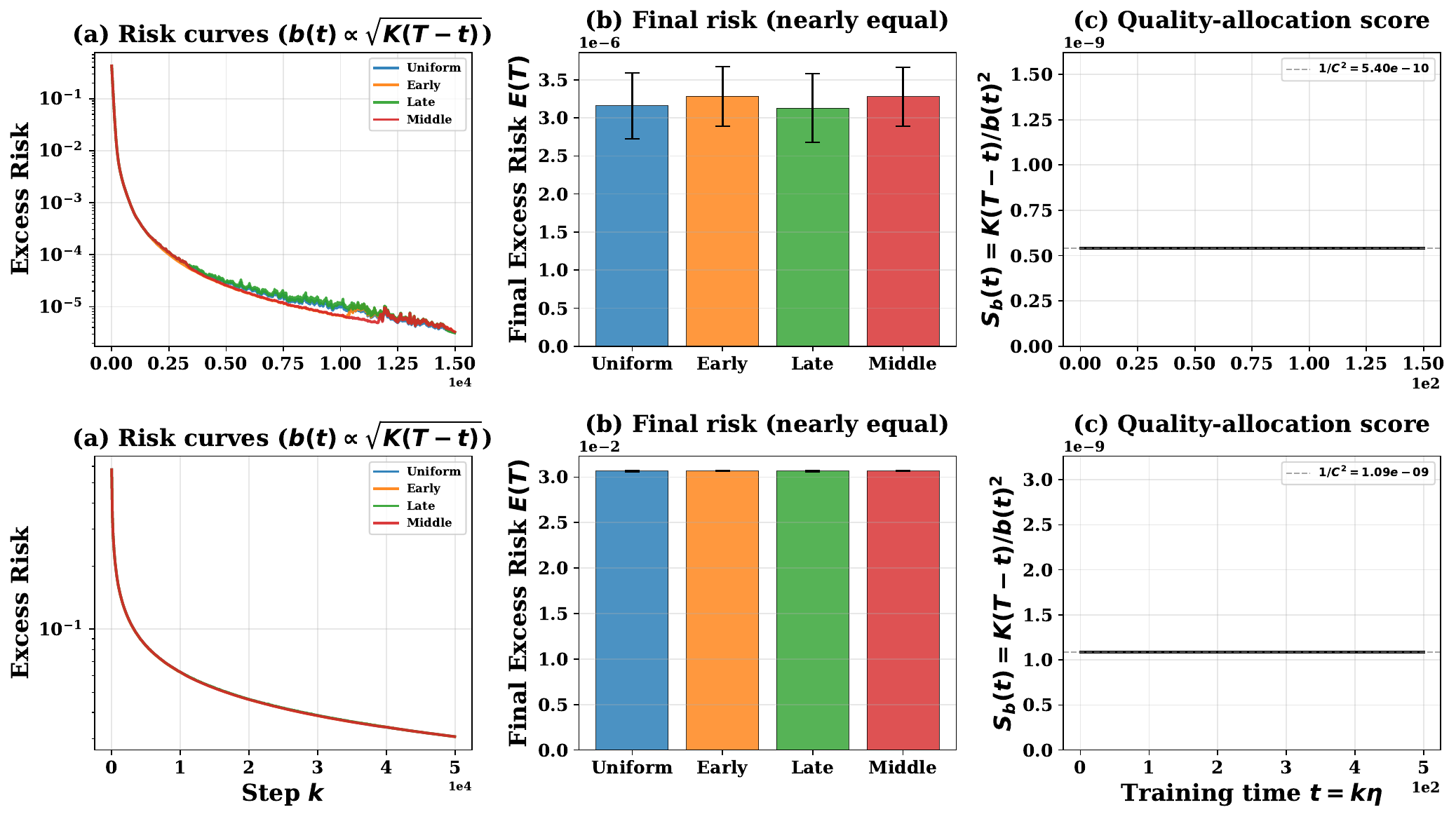}
    \caption{Corollary~\ref{cor:quality-shape}(ii): batch
    \(b(t)=C\sqrt{K(T-t)}\).
    Top: noise-limited (\(s=2\)); bottom: signal-limited (\(s=0.4\)).
    (a)~Risk curves for the four schedules are nearly identical.
    (b)~Final excess risks are equal within standard error.
    (c)~Score \(S_b(t)\approx 1/C^2\) is approximately constant,
    confirming degeneracy.}
    \label{fig:sqrt-K-batch}
\end{figure}

\paragraph{Optimal-batch experiment (Corollary~\ref{cor:quality-shape}(ii)).}
We use the same eigenstructure, learning rate, and quality parameters as
above, with total data budget \(D=200{,}000\). For the noise-limited setting
(\(s=2\), \(\sigma_{\bad}^2=1\), \(K=15{,}000\)) and the signal-limited setting
(\(s=0.4\), \(\sigma_{\bad}^2=10\), \(K=50{,}000\)), we replace the
constant batch with the schedule
\(B_k=\max\{C\sqrt{K(T-k\eta)},\,1\}\), where
\(C=D/\int_0^T\!\sqrt{K(T-t)}\dd t\).

Because the batch size now varies across time, a fixed number of
high-quality \emph{steps} would no longer guarantee equal high-quality
\emph{data} budgets across schedules. We therefore construct budget-aware
quality schedules: for the Early schedule, we set \(p_k=1\) on the first
\(n_{\mathrm{early}}\) steps such that
\(\sum_{k=1}^{n_{\mathrm{early}}} B_k\eta=\rho D\); for the Late schedule,
we set \(p_k=1\) on the last \(n_{\mathrm{late}}\) steps satisfying the
same budget; and similarly for Middle. This ensures that all schedules
consume exactly the same total number of high-quality samples, so any
performance difference can be attributed solely to the \emph{timing} of
quality placement.

When \(b(t)=C\sqrt{K(T-t)}\) (the rate-optimal schedule
of \citet{wang2025fastcatchup}), the score becomes
\(S_b(t)=1/C^2\), a constant independent of \(t\), so quality placement is
degenerate. Figure~\ref{fig:sqrt-K-batch} confirms this: unlike the
constant-batch case, the four quality schedules produce nearly
indistinguishable risk curves in both the noise-limited and signal-limited settings,
with final excess risks equal within standard error. Panel~(c) shows that
\(S_b(t)\) is approximately flat, confirming degeneracy.

\subsection{Experimental details for Theorem~\ref{thm:joint-optimal} verification}
\label{app:joint-experiment}

\paragraph{Parameter settings.}
We fix dimension \(N=500\), high-quality fraction \(\rho=0.3\), and
\(\sigma_{\good}^2/\sigma_{\bad}^2<1\) throughout. All results are averaged
over 10 independent random seeds.

\emph{Noise-limited regime} (\(s=2\), \(\beta=2\)): learning rate \(\eta=0.01\),
noise variances \(\sigma_{\good}^2=0.1\), \(\sigma_{\bad}^2=1.0\)
(variance ratio \(r=0.1\)), total data budget \(D=200{,}000\).

\emph{Signal-limited regime} (\(s=0.5\), \(\beta=3\)): learning rate
\(\eta=0.05\), noise variances \(\sigma_{\good}^2=0.25\),
\(\sigma_{\bad}^2=1.0\) (variance ratio \(r=0.25\)), total data budget
\(D=10{,}000\), minimum batch size \(B_{\min}=4\).

\paragraph{Strategy construction.}
All four strategies share the same total data budget \(D\) and high-quality
data budget \(\rho D\).

\begin{enumerate}[nosep]
\item \textbf{Constant $b$ + uniform $p$}: \(b(t)=D/T\), \(p(t)=\rho\).
In the noise-limited regime, the training horizon is set to the joint optimal
\(T^\star\) from Strategy~4.
In the signal-limited regime, the constant batch is \(b=B_{\min}\), which forces
\(T=D/B_{\min}\) (the maximum feasible horizon).

\item \textbf{Constant $b$ + late $p$}: same batch size and horizon as
Strategy~1; quality is bang-bang with \(p=1\) on the last \(\rho T/\eta\)
steps and \(p=0\) otherwise, satisfying the high-quality budget constraint.

\item \textbf{Uniform $p$ + optimal $b$}: \(p(t)=\rho\) throughout;
batch size \(b(t)=\max(C\sqrt{K(T\!-\!t)},\,B_{\min})\) with the constant
\(C\) determined by bisection so that \(\int_0^T b(t)\dd t=D\). The
training horizon \(T^\star\) is independently optimized by minimizing the
FSL objective
\(T^{-s}+\eta\int_0^T K(T\!-\!t)\sigma_{\mathrm{eff}}^2/b(t)\,\dd t\)
over \(T\) using bounded scalar minimization.

\item \textbf{Joint optimal}: implements Theorem~\ref{thm:joint-optimal}
directly. In the noise-limited regime (Part~I), the training horizon \(T^\star\)
is optimized over \(T\); the batch schedule is
\(b(t)=C\sqrt{K(T^\star\!-\!t)}\) on low-quality steps and
\(b(t)=rC\sqrt{K(T^\star\!-\!t)}\) on high-quality steps, with quality
placement chosen as late (one valid choice among the degenerate family).
In the signal-limited regime (Part~II), the three-phase structure is
constructed: Phase~1 (\(p=0\), \(b=B_{\min}\)) exhausts all low-quality
data over \(T_1^\star=(1-\rho)D/B_{\min}\); Phase~3 (\(p=1\),
ramp-up) uses
\(b(t)=B_{\min}\bigl((T^\star\!-\!t\!+\!1)/(T_4^\star\!+\!1)\bigr)^{1/(2\beta)-1}\);
Phase~2 (\(p=1\), \(b=B_{\min}\)) occupies the remainder
\(T_3^\star=T^\star\!-\!T_1^\star\!-\!T_4^\star\).
The ramp-up duration \(T_4^\star\) is optimized numerically by minimizing
the quality-aware FSL objective over \(T_4\), which implicitly determines the total
horizon \(T^\star=T_1^\star+T_3^\star+T_4^\star\).
\end{enumerate}

\section{LLM Experiment Details}
\label{app:llm-details}

\subsection{Model configurations}
\label{app:model-configurations}

\paragraph{15B-total/2.4B-active MoE.}
The proprietary-data experiments and the public-data MoE validation use the same shortcut-connected MoE architecture of \citet{cai2025shortcutconnected}. The model has 12 Transformer layers, a hidden dimension of 2,304, a shared FFN dimension of 4,608, 24 attention heads, and 8 KV heads for grouped-query attention. Each MoE layer contains 112 experts with an expert FFN dimension of 1,536, and the router activates the top 8 experts per token.

\paragraph{600M dense model.}
The public-data dense model has 24 Transformer layers, a hidden dimension of 1,536, a shared FFN dimension of 4,096, 12 attention heads, and 4 KV heads. It directly uses the standard HuggingFace LLaMA forward implementation~\citep{touvron2023llama2}, with RMSNorm, SwiGLU, grouped-query attention, RoPE base $10^6$, bias-free linear layers, untied input/output embeddings, and zero attention and hidden dropout.

\subsection{Data composition}
\label{app:data-composition}

\subsubsection{Proprietary-data experiments}
\label{app:proprietary-data-composition}

\paragraph{Pretraining corpus.} The proprietary-data setting uses a general pretraining corpus covering diverse topics, domains, and languages. Similar to standard web corpora, this data is inherently noisy with respect to specific downstream task distributions, corresponding to high $\sigma_{\bad}^2$ in our theoretical framework. All proprietary-data midtraining runs fork from one checkpoint pretrained on 1.2T tokens from this corpus.

\paragraph{Midtraining corpus.} The proprietary midtraining corpus is specifically curated for reasoning and coding enhancement, designed to align closely with complex downstream evaluation benchmarks. It consists of:
\begin{itemize}[nosep]
\item \textbf{Reasoning enhancement data}: high-quality question-answering pairs, chain-of-thought annotations, and step-by-step mathematical proof trajectories to bolster logical deduction and problem-solving capabilities.
\item \textbf{Coding enhancement data}: carefully filtered proprietary code datasets and high-quality repositories with documentation, covering algorithmic problem-solving across multiple programming languages.
\end{itemize}
This focused corpus exhibits significantly lower noise variance $\sigma_{\good}^2$ with respect to downstream reasoning and coding benchmarks, as it is tightly aligned in both format and distribution.

\subsubsection{Public-data experiments}
\label{app:public-data-composition}

\paragraph{Pretraining corpus.} Both the 600M dense model and the 15B-total/2.4B-active MoE model are pretrained on 300B FineWeb-Edu tokens~\citep{penedo2024fineweb}. Within each architecture, every midtraining schedule starts from the same final pretrained checkpoint. Holding the public pretraining corpus and token budget fixed across architectures isolates architectural transfer from differences in the data pipeline.

\paragraph{Midtraining corpus.} Both models are midtrained for 30B tokens on the same 1:1 public mixture. The code half contains 15B tokens from SwallowCode~\citep{fujii2026rewriting}, and the mathematics half contains 15B tokens from the \texttt{4plus} subset of Nemotron-CC-Math~\citep{mahabadi2026nemotronccmath}. Every schedule consumes the same mixture, data order, and total token budget within each architecture.

\subsection{Training and schedule configurations}

The proprietary-data and public-data experiments use different hardware configurations, with the public-data configurations incurring greater communication overhead; consequently, smaller batches attain lower MFU relative to the corresponding maximum-batch runs. Because absolute utilization measurements and accelerator identifiers are subject to infrastructure confidentiality constraints, we report only within-setting MFU ratios as reference values rather than as platform-independent efficiency measurements. MFU depends jointly on the hardware, network interconnect, batch size, model architecture, training framework, and parallelism configuration.

\subsubsection{Proprietary-data experiments}

The proprietary-data runs start from the same MoE checkpoint pretrained on 1.2T tokens of general web data. All schedules use a 108B-token midtraining budget, sequence length 8,192, BF16 precision, and AdamW~\citep{loshchilov2018decoupled} with $\beta_1{=}0.9$, $\beta_2{=}0.95$, $\epsilon{=}10^{-12}$, weight decay $0.1$, peak learning rate $6.51\times10^{-4}$, and minimum learning rate $4.2\times10^{-5}$. Steps are reported in units of 1k.

\paragraph{Batch-size schedule comparison (the first two sets of Table~\ref{tab:main_results}).} We compare three schedules on both the pretraining and midtraining corpora, and add Constant~2k as a control in the midtraining setting. All runs use identical token budgets:
\begin{itemize}[nosep]
\item \textbf{Constant 4k}: global batch size held at 4096 throughout the 108B midtraining span ($\approx 3.2\text{k}$ steps).
\item \textbf{Ramp-up 4k$\to$8k}: batch size steps through $\{4096,5120,6144,7168\}$ over the first 86B tokens (Megatron's equal-sample ramp with increment 1024), then stays at 8192 for the remaining 22B ($\approx 2.3\text{k}$ steps).
\item \textbf{Drop 1k$\to$8k}: batch size drops to 1024 at the start of midtraining, steps through $\{1024, 2048, \dots, 7168\}$ over 94B tokens, then stays at 8192 for the remaining 14B ($\approx 4.3\text{k}$ steps).
\item \textbf{Constant 2k (midtraining only)}: global batch size held at 2048 throughout the 108B-token span ($\approx 6.4\text{k}$ steps).
\end{itemize}
All schedules consume the same total tokens. Constant~2k has the largest total number of optimizer updates, while Drop concentrates additional updates in the early phase when the model first encounters the midtraining distribution shift. This distinction directly tests whether the gain comes from the timing of small-batch training rather than from more updates alone.

\paragraph{Stable-phase ratio experiment (Figure~\ref{fig:stable_ratio}).} To finely sweep the duration of the low-batch stable phase, we instantiate the DSR recipe with batch-size bounds $[512, 4096]$ and vary the stable-phase ratio:
\begin{itemize}[nosep]
\item \textbf{Drop}: at the start of midtraining, the batch size drops from the pretraining value of 4096 to 512.
\item \textbf{Stable}: batch size held at 512 with learning rate fixed at $\eta$.
\item \textbf{Rampup}: batch size steps through $\{512,1024,\dots,3584\}$ in equal-sample increments of 512 with learning rate fixed at $\eta$, then held at 4096 with learning rate linearly decayed from $\eta$ to $4.2\times10^{-5}$ (equivalent to continuing the batch-size ramp-up by~\citet{smith2018dont}).
\end{itemize}
We fix the ramp-up and decay phases to receive equal token shares of the midtraining budget after the stable phase, and sweep the stable-phase ratio from 0\% to 62.14\% (corresponding to stable-phase entry points of $\{0, 3200, 6400, 9600, 12800, 16000\}$ steps in the parent run).

\paragraph{Schedule comparison (the last set of Table~\ref{tab:main_results}).} Three schedules are compared:
\begin{itemize}[nosep]
\item \textbf{Cosine-decay}: Constant batch size 4096, learning rate follows a cosine decay schedule from $6.51\times10^{-4}$ to $4.2\times10^{-5}$.
\item \textbf{WSD}~\citep{hu2024minicpm}: Constant batch size 4096, learning rate held at $6.51\times10^{-4}$ for the first 72B tokens, then linearly decayed to $4.2\times10^{-5}$ over the final 36B tokens.
\item \textbf{Drop-Stable-Rampup}: Our proposed schedule with batch size [512, 4096], 24.9\% stable ratio, followed by linear ramp-up and learning rate decay.
\end{itemize}
All three schedules consume the same 108B token budget.

\paragraph{Relative efficiency and wall-clock cost.}
The proprietary-data MoE experiments are conducted on 384 accelerators using the Megatron framework~\citep{shoeybi2019megatron}, configured with pipeline parallelism $pp=6$ and tensor parallelism $tp=1$. Batch size 512 achieves $0.78\times$ the MFU at batch size 4,096. Accounting for the measured efficiency of each batch-size bucket and its token allocation, DSR requires $1.10\times$ the wall-clock time of a constant-4,096 run under the same 108B-token budget and hardware allocation. This overhead reflects reduced hardware utilization rather than additional algorithmic FLOPs.

\subsubsection{Public-data experiments}
\label{app:public-validation-details}

We evaluate the 600M dense model and the 15B-total/2.4B-active MoE model described in Appendix~\ref{app:model-configurations} using the shared public corpora detailed in Appendix~\ref{app:public-data-composition}.

\paragraph{Training configurations.}
Both architectures use sequence length 8,192, BF16 precision, micro-batch size 1, random seed 48, gradient clipping at 1.0, and AdamW with $\beta_1{=}0.9$, $\beta_2{=}0.95$, and weight decay 0.1. The dense model is trained on 128 accelerators and uses $\epsilon{=}10^{-8}$, peak/minimum learning rates $3.0\times10^{-4}$/$3.0\times10^{-5}$, and no learning-rate warmup. The MoE model uses $\epsilon{=}10^{-12}$, peak/minimum learning rates $6.51\times10^{-4}$/$4.2\times10^{-5}$, and a 10-step warmup shared by all MoE baselines. For the MoE runs, the load-balancing and device-balancing coefficients are both 0.001, the router z-loss coefficient is 0.1, tensor parallelism is 1, and pipeline parallelism is 6. Within each architecture, all schedules use identical data and data order, tokenizer, optimizer, random seed, evaluation checkpoints, and evaluation protocol.

\paragraph{Schedule comparison.}
For both models, we compare Cosine, WSD-Early, WSD-Late, and Drop-Stable-Rampup under the same 30B-token budget. Cosine decays throughout all 30B tokens. WSD-Early also begins learning-rate decay at 0B, whereas WSD-Late holds the peak learning rate for 20B tokens and decays over the final 10B. DSR divides midtraining into three equal 10B-token phases: a minimum-batch stable phase, a batch-size ramp, and a maximum-batch phase with terminal learning-rate decay. The dense model holds the global batch size at 256 for the first 10B tokens, ramps it from 256 to 1,024 over the next 10B, and holds it at 1,024 during the final 10B-token learning-rate decay. The MoE model analogously uses batch sizes 1,024 and 4,096. Only the batch-size and learning-rate schedules differ among methods within an architecture.

\paragraph{Relative efficiency and wall-clock cost.}
For each public-data model, DSR and its constant-batch reference use the same hardware configuration and allocation. For the public-data MoE model, batch size 1,024 achieves $0.81\times$ the MFU at batch size 4,096. Accounting for the measured efficiency of each batch-size bucket and its token allocation, DSR requires $1.11\times$ the wall-clock time of a constant-4,096 run under the same 30B-token budget. The 600M dense model is trained on 128 accelerators; batch size 256 achieves $0.67\times$ the MFU at batch size 1,024. Its DSR run similarly requires $1.23\times$ the wall-clock time of a constant-1,024 run under the same token budget. These ratios are provided only as within-setting reference values and are not compared across hardware configurations.

\subsection{Evaluation benchmarks}

\subsubsection{Proprietary-data experiments}

We evaluate on 14 benchmarks spanning four categories:
\begin{itemize}[nosep]
\item \textbf{Knowledge}: MMLU~\citep{hendrycks2021measuring}, MMLU-Pro~\citep{wang2024mmlu}, CMMLU~\citep{li2024cmmlu}, C-Eval~\citep{huang2023ceval}.
\item \textbf{Reasoning}: ARC-Challenge~\citep{clark2018think}, BIG-Bench Hard (BBH)~\citep{suzgun2023challenging}.
\item \textbf{Math}: GSM8K~\citep{cobbe2021training}, MATH~\citep{hendrycks2021math}.
\item \textbf{Code}: HumanEval+~(HE+)~\citep{liu2023your}, BigCodeBench~(BCB)~\citep{zhuo2025bigcodebench}, LiveCodeBench~(LCB)~\citep{jain2025livecodebench}, CRUX-I, CRUX-O~\citep{gu2024cruxeval}, MultiPL-E~\citep{cassano2023multiple}.
\end{itemize}
All benchmarks are evaluated in zero-shot or few-shot settings following standard protocols. For MMLU, CMMLU, and C-Eval, we report 5-shot accuracy. For GSM8K and MATH, we use 8-shot chain-of-thought prompting. For code benchmarks, we use pass@1 with greedy decoding.

\subsubsection{Public-data experiments}

We evaluate STEM knowledge on MMLU-STEM and MMLU-Pro-STEM; mathematical reasoning on GSM8K and MATH with chain-of-thought prompting; code generation on HumanEval+, LiveCodeBench, MBPP+, BigCodeBench, and MultiPL-E; and general reasoning on BBH and ARC-Challenge. No reportable LiveCodeBench score was obtained for the dense model, so this benchmark is shown with dashes and excluded from the dense-model mean. HumanEval+ and MBPP+ use a 3-shot protocol. The prompts, decoding settings, and evaluation implementation are held fixed across schedules within each architecture.

% Signal-Limited vs. Noise-Limited Benchmark Analysis moved to main text (sections/analysis.tex)

\section{Heterogeneous training dynamics across benchmarks}
\label{sec:regime-analysis}

Our theory predicts that the optimal schedule depends on whether training is signal-limited or noise-limited, determined by the exponents $s$ and $\beta$ in the linear regression model. In practice, we do not have direct access to these parameters for LLM training, nor do we know a priori whether different tasks exhibit uniform or divergent dynamics. In this section, we analyze the accuracy trajectories of 14 benchmarks and report three findings.

First, different benchmarks exhibit clearly divergent dynamics during the stable phase. We define the \emph{early-phase gap} as the average accuracy difference between WSD and Drop-Stable-Rampup over all checkpoints within the stable phase (Figure~\ref{fig:app-all-curves}). Since the stable phase uses a much smaller batch size (512 vs.\ 4096), all benchmarks suffer from increased gradient noise and DSR generally trails WSD during this phase. However, the magnitude of this deficit varies substantially across benchmarks: some show a large early-phase gap (DSR falls far behind WSD), while others show a negligible gap (DSR remains close to WSD despite the small batch size).

Second, consistent with our theoretical prediction, benchmarks with smaller early-phase gaps achieve larger final gains from Drop-Stable-Rampup. A small early gap indicates that the increased noise from reducing the batch size is largely offset by the signal accumulated through additional gradient steps---this signal remains latent during the stable phase but manifests after the ramp-up. Conversely, a large early gap indicates that the noise cost of small batch size dominates, and the ramp-up phase can only partially recover the deficit.

Third, from the perspective of task content, benchmarks with small early-phase gaps correspond to mathematical reasoning (GSM8K, MATH) and algorithmic code generation (HumanEval+, MultiPL-E)---capabilities that require compositional, multi-step inference. A notable exception is BigCodeBench: despite being a code benchmark, it exhibits a large early-phase gap. BigCodeBench emphasizes diverse library/function calls and complex instructions; in our trajectories, its behavior suggests that the tested capabilities may be more sensitive to API/library knowledge than to the compositional signal accumulation observed in HumanEval+ and MultiPL-E.

\begin{figure*}[ht]
    \centering
    \includegraphics[width=\linewidth]{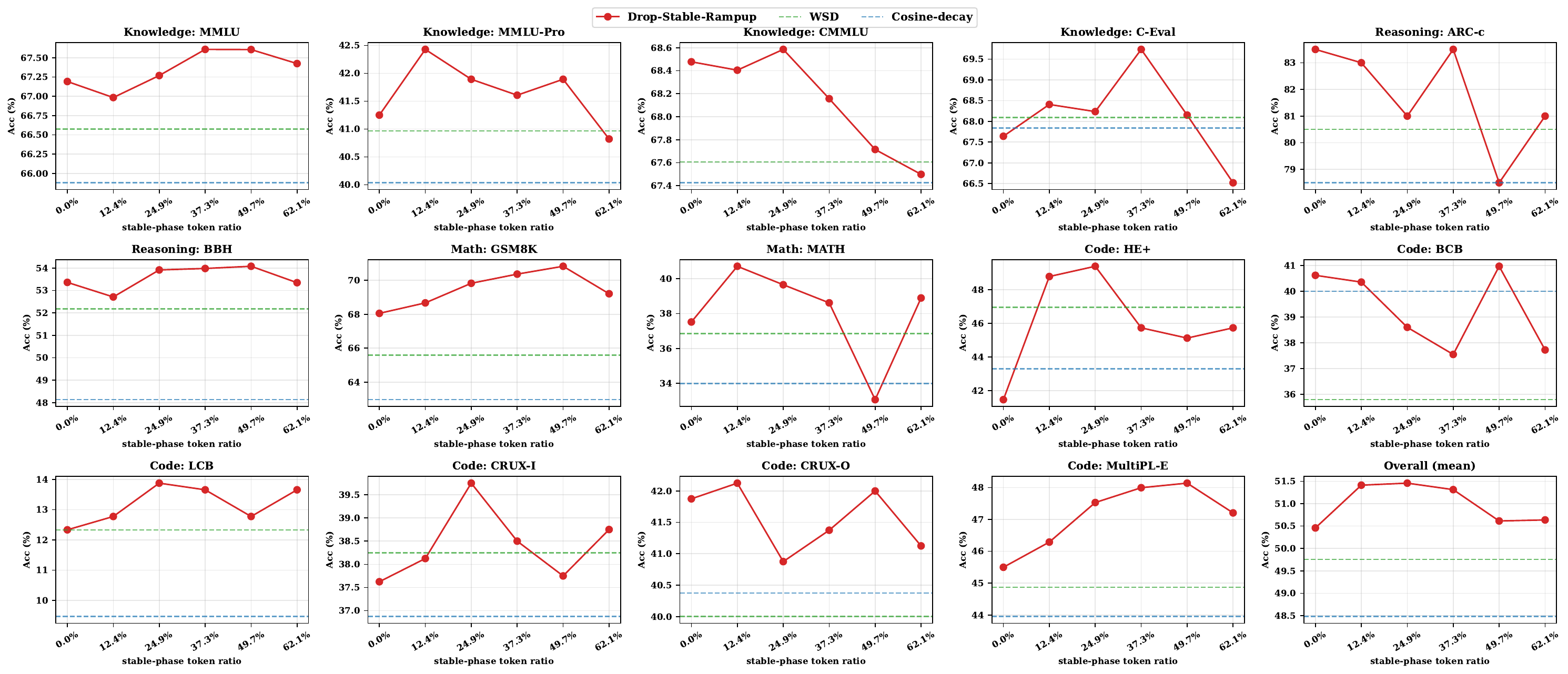}
    \caption{Final checkpoint evaluation results across different schedule configurations, demonstrating the robustness of the Drop-Stable-Rampup strategy.}
    \label{fig:app-rampup-ckpt}
\end{figure*}

\begin{figure*}[ht]
    \centering
    \includegraphics[width=\linewidth]{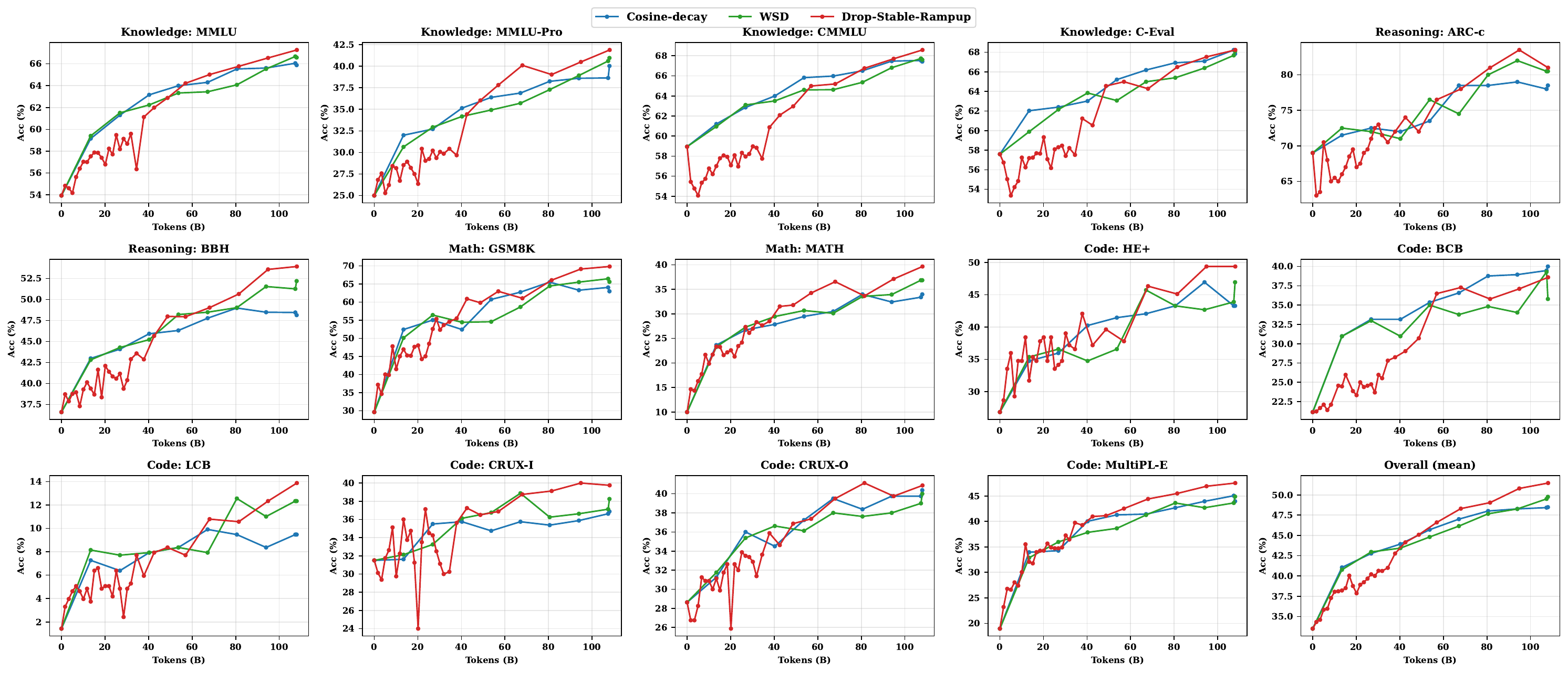}
    \caption{Per-benchmark accuracy curves for Cosine-decay, WSD, and Drop-Stable-Rampup across all 14 evaluation benchmarks.}
    \label{fig:app-all-curves}
\end{figure*}

\paragraph{Benchmarks with small early-phase gaps.}
These benchmarks show that DSR remains close to WSD during the stable phase despite using a much smaller batch size (Figure~\ref{fig:app-all-curves}). The noise penalty of small batch size is largely absorbed, and DSR overtakes WSD during the ramp-up/decay stage, ultimately achieving the largest final gains. This is consistent with the interpretation that extra gradient steps at low batch size accumulate latent capacity that only manifests once noise is suppressed during ramp-up.

\begin{itemize}[nosep]
\item \textbf{GSM8K} (+4.23): Tests multi-step arithmetic reasoning requiring chaining 2--8 sequential operations. The training curve shows near-zero early gap, with a sharp separation emerging only during the ramp-up/decay phase.
\item \textbf{MATH} (+2.80): Probes competition-level mathematical problem-solving across prealgebra, algebra, number theory, counting and probability, geometry, intermediate algebra, and precalculus. The curve shows minimal early gap, with gains accumulating gradually during ramp-up.
\item \textbf{HumanEval+} (+2.44): Evaluates Python code generation with rigorous test cases. Virtually no early-phase gap; the performance separation emerges only in the later ramp-up/decay stage.
\item \textbf{MultiPL-E} (+2.66): Assesses code generation across multiple programming languages. Small early gap combined with large final gain.
\end{itemize}

\paragraph{Benchmarks with large early-phase gaps.}
These benchmarks show that DSR falls substantially behind WSD during the stable phase. The ramp-up phase partially closes this deficit, but the final gains are smaller than the small-gap group, and in one case DSR underperforms a baseline.

\begin{itemize}[nosep]
\item \textbf{CMMLU} (+0.98): The Chinese counterpart of MMLU, it shows the most pronounced early-phase gap among all benchmarks---DSR trails WSD by a large margin throughout the stable phase, and only partially recovers during ramp-up.
\item \textbf{BigCodeBench (BCB)} (+2.81 over WSD, but $-1.40$ vs.\ Cosine-decay): Despite being a code benchmark, BCB exhibits a very large early-phase gap. BigCodeBench emphasizes practical code generation with diverse library/function calls and complex instructions. In our trajectories, its large early gap suggests that the tested capabilities may be more sensitive to API/library knowledge and function-call conventions than to the delayed signal accumulation observed in algorithmic code-generation benchmarks such as HumanEval+ and MultiPL-E. Notably, BCB is the only benchmark where Drop-Stable-Rampup underperforms Cosine-decay, suggesting that for tasks with this trajectory pattern, the extended low-batch-size phase is unnecessary and a smoother schedule may be preferable.
\end{itemize}

\paragraph{Benchmarks with moderate early-phase gaps.}
Several benchmarks fall between the two groups, exhibiting a moderate deficit during the stable phase that is largely recovered during ramp-up.

\begin{itemize}[nosep]
\item \textbf{MMLU} (+0.69): Measures broad factual knowledge across 57 subjects. A moderate early-phase gap, with the final gain being modest but positive.
\item \textbf{MMLU-Pro} (+0.93): An extended version of MMLU with harder distractors and 10-way multiple choice. Similar moderate early gap and final gain.
\item \textbf{BBH} (+1.73): Tests diverse compositional reasoning (logical deduction, causal reasoning, disambiguation). The moderate gap likely reflects a mixture of subtask types with different noise sensitivities.
\end{itemize}

\paragraph{Summary.}
The key empirical finding is a negative correlation between the early-phase gap (how much DSR trails WSD during the stable phase) and the final gain from Drop-Stable-Rampup. Benchmarks where the small batch size causes little performance deficit during the stable phase (GSM8K, MATH, HumanEval+, MultiPL-E) ultimately achieve the largest final improvements (+2.44 to +4.23). These correspond to compositional reasoning and algorithmic code generation. Benchmarks where the small batch size causes a large deficit (CMMLU, BCB) achieve smaller or negative final gains---the noise cost of the extended stable phase is not fully compensated by signal accumulation. This pattern is consistent with our theoretical prediction: tasks where signal accumulation is the bottleneck tolerate the noise of small batch size and benefit from extra gradient steps; tasks where noise dominates are harmed by the extended low-batch-size phase.

\paragraph{Implications for future work.}
The divergence in dynamics across tasks suggests a natural direction: \emph{task-aware data scheduling}. If the early-phase gap can be estimated from a small pilot run, one could tailor the data mixture to the schedule phase---prioritizing compositional data (math, code, chain-of-thought) during the stable phase to maximize signal accumulation, and shifting toward distributional content (world knowledge, encyclopedic text) during ramp-up. We leave the formal optimization of such per-domain scheduling as future work.

\end{document}